\documentclass[11pt]{article}
\PassOptionsToPackage{nopatch=footnote}{microtype}
\usepackage[letterpaper,margin=1in]{geometry}
\usepackage[utf8]{inputenc}      
\usepackage[T1]{fontenc}
\usepackage{lmodern}
\usepackage{microtype}

\usepackage[english]{babel}
\usepackage{indentfirst}

\usepackage{amsmath,amssymb,amsthm,mathtools}

\usepackage{graphicx}
\usepackage{booktabs}
\usepackage{caption}
\usepackage{subcaption}

\RequirePackage{tikz}
\usetikzlibrary{arrows.meta,positioning,fit,calc}
\usetikzlibrary{positioning,fit,calc,decorations.pathreplacing, backgrounds}
\usepackage{adjustbox}
\usetikzlibrary{arrows.meta,positioning,calc,shadows.blur}

\usepackage[hidelinks]{hyperref}
\usepackage[nameinlink,noabbrev]{cleveref}

\usepackage{listings}
\usepackage{enumitem}

\usepackage{tabularx}

\usepackage{xurl} 

\usepackage{algorithm}
\usepackage{algpseudocode} 

\theoremstyle{plain}
\newtheorem{theorem}{Theorem}[section]

\theoremstyle{definition}

\theoremstyle{remark}

\newtheorem{proposition}[theorem]{Proposition}
\newtheorem{architecture}{Architecture}

\usepackage[alf]{abntex2cite}

\usepackage{ragged2e}
\usepackage{etoolbox}

\makeatletter
\patchcmd{\thebibliography}{\raggedright}{}{}{}
\patchcmd{\thebibliography}{\RaggedRight}{}{}{}
\patchcmd{\thebibliography}{\sloppy}{}{}{}

\apptocmd{\thebibliography}{%
  \justifying
  \fussy
  \setlength{\emergencystretch}{2em}%
}{}{}
\makeatother

\title{From Shallow Bayesian Neural Networks to Gaussian Processes: General Convergence, Identifiability and Scalable Inference}
\author{
  Gracielle Antunes de Araújo\thanks{Universidade do Estado de Minas Gerais (UEMG), FaEnge -- Faculdade de Engenharia. \texttt{gracielle.araujo@uemg.br}} \and
  Flávio B. Gonçalves\thanks{Universidade Federal de Minas Gerais (UFMG), ICEx -- Instituto de Ciências Exatas. \texttt{fbgoncalves@ufmg.br}}
}
\date{} 

\begin{document}
\maketitle

\begin{abstract}
In this work, we study scaling limits of shallow Bayesian neural networks (BNNs) via their connection to Gaussian processes (GPs), with an emphasis on statistical modeling, identifiability, and scalable inference. We first establish a general convergence result from BNNs to GPs by relaxing assumptions used in prior formulations, and we compare alternative parameterizations of the limiting GP model. Building on this theory, we propose a new covariance function defined as a convex mixture of components induced by four widely used activation functions, and we characterize key properties including positive definiteness and both strict and practical identifiability under different input designs. For computation, we develop a scalable maximum a posterior (MAP) training and prediction procedure using a Nyström approximation, and we show how the Nyström rank and anchor selection control the cost-accuracy trade-off. Experiments on controlled simulations and real-world tabular datasets demonstrate stable hyperparameter estimates and competitive predictive performance at realistic computational cost.
\end{abstract}

\section{Introduction}

Mobile devices and connected systems continuously generate large volumes of complex and weakly structured data, fostering the use of machine learning models in prediction and decision-making tasks. In this setting, neural networks (NNs) stand out for their strong empirical performance, yet their increasing complexity often reduces interpretability and hampers rigorous uncertainty quantification—an issue that becomes critical in applications requiring calibration, robustness, and reliability \cite{lecun2015deep,Goodfellow-et-al-2016}.

A natural way to incorporate uncertainty is Bayesian modeling, which treats weights and biases as random variables and yields predictive distributions equipped with explicit measures of uncertainty. For shallow BNNs, classical results show that, under standard initialization and scaling assumptions, the infinite-width regime converges to a GP whose kernel is determined by the activation function and the prior distributions over the parameters \cite{neal1996bayesian,williams1997computing}. This connection provides a useful statistical bridge between NN models and stochastic processes, but it leaves open practical choices in modeling and inference when one seeks both kernel flexibility and computational tractability at scale.

From a modeling perspective, much of the literature still focuses on isolated activations or restricted families, which limits the diversity of limiting covariance structures and may constrain the ability to capture, in an interpretable way, combined effects of smoothness and angular behavior commonly observed in real data \cite{lee2017deep,novak2018bayesian}. From a computational perspective, even with an expressive kernel, exact GP inference scales cubically in time and quadratically in memory with the sample size, which makes direct use prohibitive for large datasets \cite{rasmussen2006gaussian}. Thus, exploiting the network--process equivalence brings a recurring tension between kernel expressiveness, interpretability/uncertainty quantification, and computational cost.

This paper studies scaling limits of shallow BNNs via their connection to GPs, with emphasis on modeling, identifiability, and scalable inference.
First, we establish a general convergence result from wide BNNs to GPs under relaxed assumptions.
Second, we introduce a new covariance function based on a convex mixture of activation-induced components (tanh, sigmoid, ReLU, and Leaky-ReLU), yielding an interpretable kernel that combines smooth and angular behavior.
Third, we study positive definiteness and both strict and practical identifiability as a function of the input design, highlighting regimes where different mixtures become difficult to distinguish from finite samples.
Finally, we develop a scalable MAP training and prediction procedure via a Nyström approximation, making the cost--accuracy trade-off explicit through the rank and anchor-selection strategy.

We evaluate the proposed methodology on simulations and real-world tabular datasets, emphasizing estimation stability, predictive performance, and scalability as a function of the Nyström rank and anchor selection.

The remainder of the article is organized as follows. Section~\ref{sec-review} reviews the connection between shallow BNNs and GPs and introduces notation. Section~\ref{sec:shallow_bnns_gp} establishes the infinite-width limit and the general convergence result. Section~\ref{sec:mixed_kernel_bnn} derives the mixed kernel, discusses its properties, and presents identifiability results. Section~\ref{sec:estimacao_map_nystrom} describes the Nyström-based MAP procedure for scalable training and prediction. Section~\ref{sec:experimentos} presents the experiments and discusses the empirical findings. Finally, Section~\ref{sec:conclusao} summarizes the conclusions and outlines directions for future work.

\subsection{Background and Related Work}\label{sec-review}
Neural networks provide flexible function approximation and strong predictive performance, but increasing model complexity can hinder interpretability and rigorous uncertainty quantification. A Bayesian formulation addresses these issues by placing priors on weights and biases and producing predictive distributions with explicit measures of uncertainty \cite{mackay1997gaussian,wilson2020bayesian}. This perspective is particularly relevant in applied machine learning settings where decisions must be accompanied by calibrated notions of reliability.

A central theoretical result is that shallow BNNs, under standard scaling and independence assumptions, converge in the infinite-width limit (in the sense of finite-dimensional distributions) to Gaussian processes \cite{neal1996bayesian}. For specific architectures, closed-form expressions for the induced covariance function can be derived, completing the description of the limiting GP model \cite{williams1997computing}. The broader theory and exact inference machinery for GPs are consolidated in \cite{rasmussen2006gaussian}, providing a unified statistical framework for probabilistic prediction.

For deeper architectures, this correspondence extends through layer-wise recursions and is often described through the Neural Network Gaussian Process (\(NNGP_\infty\)) and the Neural Tangent Kernel (NTK), which clarify how activation, scaling, and initialization shape inductive bias and regularity, and how training dynamics can be characterized in large-width regimes \cite{lee2017deep,matthews2018gaussian,novak2018bayesian,bahri2018deep,jacot2018neural}. From a modeling standpoint, however, much of the literature focuses on isolated activation choices or restricted kernel families, limiting the range of covariance structures available for capturing combined smooth and angular behaviors that can arise in practice \cite{lee2017deep,novak2018bayesian}. Moreover, practical deployment requires scalable inference: exact GP computations scale cubically in time and quadratically in memory with the sample size, motivating scalable approximations such as low-rank methods and sparse neighbor-based constructions \cite{rasmussen2006gaussian,sauer2023vecchia,Datta2016NNGP,quiroz2023blocknngp,wilson2020bayesian,dutordoir2023deep}.

Two additional considerations become important when moving from limit theory to practice. First, finite-width networks exhibit residual dependencies that can affect effective kernels and uncertainty calibration \cite{vladimirova2021dependence,liu2023wide,garcia2023deep}. Second, as one adopts richer kernel parameterizations, identifiability and numerical stability become central, since distinct components may be difficult to distinguish under realistic input designs and finite samples \cite{rasmussen2006gaussian,Kuss2006}. These factors motivate methodological development that jointly addresses kernel expressiveness, identifiability, and scalable inference.

In this work, we build on these foundations by establishing a general convergence result from shallow Bayesian neural networks to Gaussian processes under weaker assumptions \cite{neal1996bayesian,williams1997computing}, proposing an interpretable mixed covariance function as a convex combination of activation-induced components with theoretical guarantees, and developing a scalable MAP training and prediction pipeline based on a Nyström approximation, where rank and anchor selection transparently govern the cost–accuracy trade-off.


\section{Infinite-width limit for shallow BNN}
\label{sec:shallow_bnns_gp}

The Bayesian approach to feedforward neural networks assigns prior distributions to model parameters,
thereby inducing a prior distribution over functions. A classical result shows that, when the number of
neurons in one (or more) hidden layers tends to infinity (\(H\to\infty\)), the induced prior over functions
converges to a GP, whose covariance is jointly determined by the activation function and
the parameter priors \cite{neal1996bayesian,williams1997computing,novak2020neural}.

\subsection{Model: shallow BNN with one hidden layer}
\label{subsec:shallow_bnn_model}

\begin{architecture}[Feedforward network with one hidden layer]
\label{def:ffn-1layer}
Let $\mathcal X\subseteq\mathbb R^{I}$ and $\mathcal Y\subseteq\mathbb R^{K}$.
Fix $H\in\mathbb N$ (number of hidden units) and activation functions
$h:\mathbb R\to\mathbb R$ (hidden layer) and $g:\mathbb R\to\mathbb R$ (output layer),
both applied componentwise.
We define the network $\mathbf f:\mathcal X\to\mathcal Y$ with parameters
\[
\mathbf U\in\mathbb R^{H\times I},\ \mathbf a\in\mathbb R^{H},\ 
\mathbf V\in\mathbb R^{K\times H},\ \mathbf b\in\mathbb R^{K}.
\]
Given an input $\mathbf x\in\mathcal X$, the forward pass is
\[
\mathbf z=\mathbf U\mathbf x+\mathbf a\in\mathbb R^{H},
\qquad
\mathbf h=h(\mathbf z)\in\mathbb R^{H},
\qquad
\mathbf f(\mathbf x)=g\!\big(\mathbf V\mathbf h+\mathbf b\big)\in\mathbb R^{K},
\]
where $g(\cdot)$ is applied elementwise in $\mathbb R^{K}$.
In indexed form, for $k=1,\dots,K$,
\[
f_k(\mathbf x)
=
g\!\left(
b_k+\sum_{j=1}^{H} v_{kj}\,
h\!\left(a_j+\sum_{i=1}^{I}u_{ji}x_i\right)
\right).
\]
For regression, we take $g(\cdot)\equiv \mathrm{Id}(\cdot)$. Hence,

\begin{equation}
\mathbf f(\mathbf x)=\mathbf V\mathbf h+\mathbf b\in\mathbb R^{K},
\qquad
f_k(\mathbf x)=b_k+\sum_{j=1}^{H} v_{kj}\,
h\!\left(a_j+\sum_{i=1}^{I}u_{ji}x_i\right).
\end{equation}

\end{architecture}
This architecture is a fully-connected shallow multilayer perceptron (MLP), i.e., a feedforward network with a single hidden layer.

Figures~\ref{fig:arquiteturaum} and~\ref{fig:forward_coluna_empilha_soB} provide two complementary views of the same one-hidden-layer feedforward network, as defined in Architecture~\ref{def:ffn-1layer}.
In Figure~\ref{fig:arquiteturaum}, an input $\mathbf x=(x_1,\ldots,x_I)$ feeds all $H$ hidden units simultaneously: each neuron $j$ computes the pre-activation
$z_j=a_j+\sum_{i=1}^I u_{ji}x_i$
and applies the nonlinearity $h(\cdot)$, yielding $h_j=h(z_j)$.
The hidden activations are then linearly combined at the output layer: for each output component $k$, the model sums the contributions $v_{kj}h_j$ and adds the bias $b_k$, resulting in
$f_k=b_k+\sum_{j=1}^H v_{kj}h_j$
(in regression, with $g(\cdot)\equiv \mathrm{Id}$).
Figure~\ref{fig:forward_coluna_empilha_soB} presents the same computation from a dataset perspective: one selects a column $\mathbf x^{(p)}$ from the input matrix $\mathbf X$, propagates this observation through the network to obtain $\mathbf f(\mathbf x^{(p)})\in\mathbb R^K$, and by repeating this for $p=1,\ldots,n$, stacks the outputs to form
$\mathbf F=[\mathbf f(\mathbf x^{(1)})\ \cdots\ \mathbf f(\mathbf x^{(n)})]\in\mathbb R^{K\times n}$.

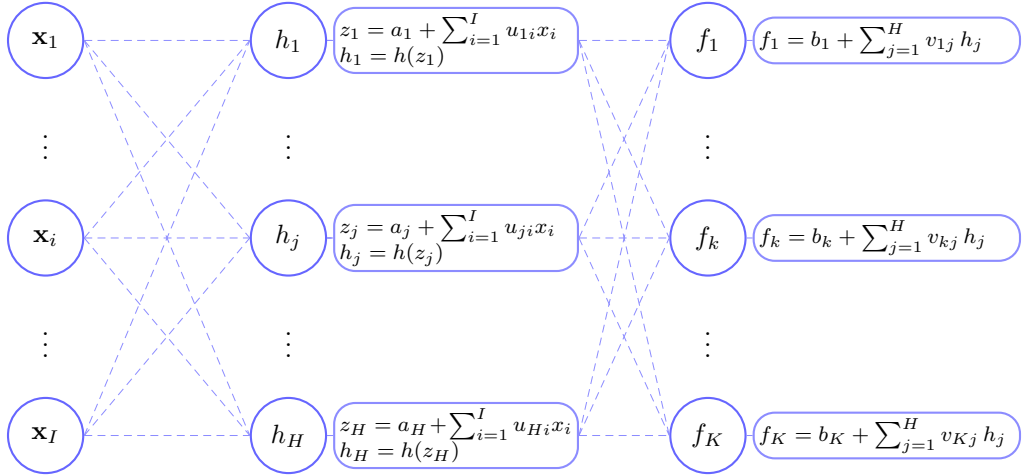
\begin{figure}[htbp]
\centering
\caption{Shallow feedforward network (one-hidden-layer MLP): $I$ inputs, $H$ hidden units, and $K$ outputs. The figure illustrates the regression case, where $g(\cdot)\equiv \mathrm{Id}(\cdot)$.}
\label{fig:arquiteturaum}

\begin{tikzpicture}[
  font=\small,
  node distance=12mm,
  xnode/.style={circle, draw=blue!60, thick, fill=white, minimum size=10mm},
  hnode/.style={circle, draw=blue!60, thick, fill=white, minimum size=10mm},
  fnode/.style={circle, draw=blue!60, thick, fill=white, minimum size=10mm},
  faint/.style={densely dashed, blue!45, line width=0.35pt},
  cardH/.style={rounded corners=8pt, draw=blue!45, thick, fill=white,
                align=left, inner sep=2pt, text width=31mm},
  cardF/.style={rounded corners=8pt, draw=blue!45, thick, fill=white,
                align=left, inner sep=2pt, text width=34mm}
]

\node[xnode] (x1) {$\mathbf{x}_1$};
\node[xnode, below=16mm of x1] (xi) {$\mathbf{x}_i$};
\node[xnode, below=16mm of xi] (xI) {$\mathbf{x}_I$};

\node at ($(x1)!0.5!(xi)$) {$\vdots$};
\node at ($(xi)!0.5!(xI)$) {$\vdots$};

\node[hnode, right=22mm of x1] (h1) {$h_1$};
\node[hnode, right=22mm of xi] (hj) {$h_j$};
\node[hnode, right=22mm of xI] (hH) {$h_H$};

\node at ($(h1)!0.5!(hj)$) {$\vdots$};
\node at ($(hj)!0.5!(hH)$) {$\vdots$};

\node[cardH, anchor=west] (c1) at ($(h1.east)+(0.8mm,0)$) {\scriptsize
$z_1=a_1+\sum_{i=1}^{I}u_{1i}x_i$\\[-2pt]
$h_1=h(z_1)$
};

\node[cardH, anchor=west] (cj) at ($(hj.east)+(0.8mm,0)$) {\scriptsize
$z_j=a_j+\sum_{i=1}^{I}u_{ji}x_i$\\[-2pt]
$h_j=h(z_j)$
};

\node[cardH, anchor=west] (cH) at ($(hH.east)+(0.8mm,0)$) {\scriptsize
$z_H=a_H+\sum_{i=1}^{I}u_{Hi}x_i$\\[-2pt]
$h_H=h(z_H)$
};

\draw[blue!45, line width=0.35pt] (h1.east) -- (c1.west);
\draw[blue!45, line width=0.35pt] (hj.east) -- (cj.west);
\draw[blue!45, line width=0.35pt] (hH.east) -- (cH.west);

\draw[blue!45, line width=0.35pt] (h1.east) -- ++(0.6mm,0);
\draw[blue!45, line width=0.35pt] (hj.east) -- ++(0.6mm,0);
\draw[blue!45, line width=0.35pt] (hH.east) -- ++(0.6mm,0);

\node[fnode, right=12mm of c1] (f1) {$f_1$};
\node[fnode, right=12mm of cj] (fk) {$f_k$};
\node[fnode, right=12mm of cH] (fK) {$f_K$};

\node at ($(f1)!0.5!(fk)$) {$\vdots$};
\node at ($(fk)!0.5!(fK)$) {$\vdots$};

\node[cardF, anchor=west] (d1) at ($(f1.east)+(0.8mm,0)$) {\scriptsize
$f_1=b_1+\sum_{j=1}^{H}v_{1j}\,h_j$
};

\node[cardF, anchor=west] (dk) at ($(fk.east)+(0.8mm,0)$) {\scriptsize
$f_k=b_k+\sum_{j=1}^{H}v_{kj}\,h_j$
};

\node[cardF, anchor=west] (dK) at ($(fK.east)+(0.8mm,0)$) {\scriptsize
$f_K=b_K+\sum_{j=1}^{H}v_{Kj}\,h_j$
};

\draw[blue!45, line width=0.35pt] (f1.east) -- (d1.west);
\draw[blue!45, line width=0.35pt] (fk.east) -- (dk.west);
\draw[blue!45, line width=0.35pt] (fK.east) -- (dK.west);

\draw[blue!45, line width=0.35pt] (h1.east) -- ++(0.6mm,0);
\draw[blue!45, line width=0.35pt] (hj.east) -- ++(0.6mm,0);
\draw[blue!45, line width=0.35pt] (hH.east) -- ++(0.6mm,0);

\foreach \x in {x1,xi,xI}{
  \foreach \h in {h1,hj,hH}{
    \draw[faint] (\x.east) -- (\h.west);
  }
}

\foreach \c in {c1,cj,cH}{
  \foreach \f in {f1,fk,fK}{
    \draw[faint] (\c.east) -- (\f.west);
  }
}

\end{tikzpicture}
\end{figure}

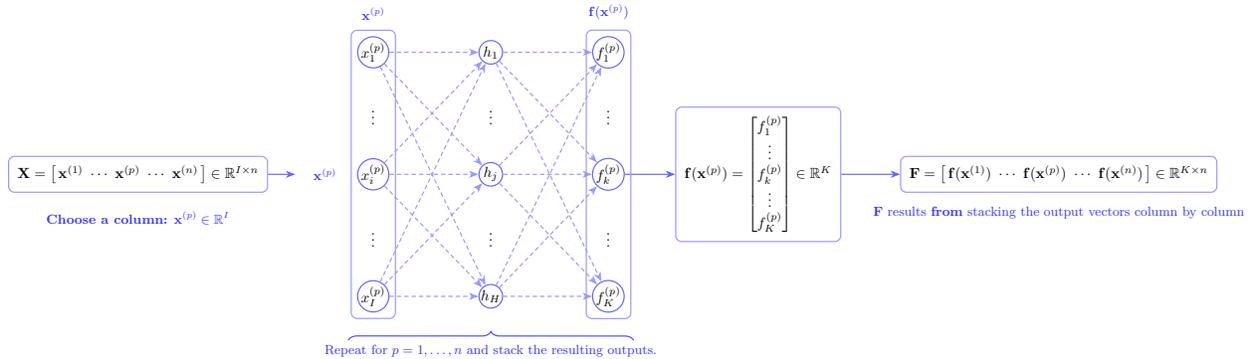
\begin{figure}[htbp]
\centering
\caption{
The column $\mathbf x^{(p)}\in\mathbb R^{I}$ is propagated through the network, producing the output vector $\mathbf f(\mathbf x^{(p)})\in\mathbb R^{K}$. Repeating this for $p=1,\dots,n$ and stacking the outputs as columns yields $\mathbf F=[\mathbf f(\mathbf x^{(1)})\ \cdots\ \mathbf f(\mathbf x^{(n)})]\in\mathbb R^{K\times n}$.
}
\label{fig:forward_coluna_empilha_soB}

\begin{adjustbox}{max width=\textwidth}
\begin{tikzpicture}[
  >=Stealth, scale=1.08,
  yscale=1.80, 
  every node/.style={font=\small},
  var/.style={circle,draw=blue!60,thick,minimum size=16pt,inner sep=0pt},
  arrow/.style={->,thick,blue!60},
  dottedarrow/.style={->,thick,blue!40,densely dashed},
  ellip/.style={draw=none,inner sep=0pt,minimum size=0pt},
  mat/.style={draw=blue!40,thick,rounded corners,inner sep=6pt}
]

\node[mat] (X) {$
\mathbf X=
\big[\,\mathbf{x}^{(1)}\ \cdots\ \mathbf{x}^{(p)}\ \cdots\ \mathbf{x}^{(n)}\,\big]
\in\mathbb{R}^{I\times n}
$};

\node[below=3mm of X, text=blue!80] (Xnote)
{\footnotesize \textbf{Choose a column:} $\mathbf{x}^{(p)}\in\mathbb{R}^{I}$};

\coordinate (Xpout) at ($(X.east)+(0.6,0)$);
\draw[arrow] (X.east) -- (Xpout);
\node[right=0.3cm of Xpout, text=blue!80] {\footnotesize $\mathbf{x}^{(p)}$};

\node[var]  (x1) at (5.2,1.5) {$x_1^{(p)}$};
\node[ellip] (xv1) at (5.2,0.75) {$\vdots$};
\node[var]  (xi) at (5.2,0) {$x_i^{(p)}$};
\node[ellip] (xv2) at (5.2,-0.75) {$\vdots$};
\node[var]  (xI) at (5.2,-1.5) {$x_I^{(p)}$};

\node[draw=blue!40, thick, rounded corners, fit=(x1)(xI), inner sep=4pt] (Xcol) {};
\node[above=1mm of Xcol, text=blue!80] {\footnotesize $\mathbf x^{(p)}$};



\node[var]  (h1)  at (7.8,1.5) {$h_1$};
\node[ellip] (hv1) at (7.8,0.75) {$\vdots$};
\node[var]  (hj)  at (7.8,0) {$h_j$};
\node[ellip] (hv2) at (7.8,-0.75) {$\vdots$};
\node[var]  (hH)  at (7.8,-1.5) {$h_H$};

\node[var]  (f1)  at (10.4,1.5) {$f_1^{(p)}$};
\node[ellip] (fv1) at (10.4,0.75) {$\vdots$};
\node[var]  (fk)  at (10.4,0) {$f_k^{(p)}$};
\node[ellip] (fv2) at (10.4,-0.75) {$\vdots$};
\node[var]  (fK)  at (10.4,-1.5) {$f_K^{(p)}$};

\foreach \a in {x1,xi,xI}{
  \foreach \b in {h1,hj,hH}{
    \draw[dottedarrow] (\a) -- (\b);
  }
}
\foreach \a in {h1,hj,hH}{
  \foreach \b in {f1,fk,fK}{
    \draw[dottedarrow] (\a) -- (\b);
  }
}

\node[mat, right=1.2cm of fk] (fpvec) {$
\mathbf{f}(\mathbf x^{(p)})=
\begin{bmatrix}
f_1^{(p)}\\ \vdots\\ f_k^{(p)}\\ \vdots\\ f_K^{(p)}
\end{bmatrix}
\in\mathbb{R}^{K}
$};

\draw[arrow] (fk.east) -- (fpvec.west);

\node[mat, right=1.4cm of fpvec] (Fmat) {$
\mathbf F=
\big[\,\mathbf{f}(\mathbf{x}^{(1)})\   \cdots\ \mathbf{f}(\mathbf{x}^{(p)})\ \cdots\ \mathbf{f}(\mathbf{x}^{(n)})\,\big]
\in\mathbb{R}^{K\times n}
$};

\node[draw=blue!40, thick, rounded corners, fit=(f1)(fK), inner sep=4pt] (Fcol) {};
\node[above=1mm of Fcol, text=blue!80] {\footnotesize $\mathbf f(\mathbf x^{(p)})$};


\draw[arrow] (fpvec.east) -- (Fmat.west);

\draw[decorate,decoration={brace,amplitude=6pt},blue!70,thick]
  ($(xI.south west)+(-0.3,-0.4)$) -- ($(fK.south east)+(0.3,-0.4)$);

\node[below=10pt, text=blue!80] at ($(xi.south)+(2.6,-1.65)$)
{\footnotesize Repeat for $p=1,\dots,n$ and stack the resulting outputs.};

\node[below=2mm of Fmat, text=blue!80]
{\footnotesize $\mathbf F$ results \textbf{from} stacking the output vectors column by column};

\end{tikzpicture}
\end{adjustbox}

\end{figure}

We assume that the observed responses $\mathbf y^{(p)}$ are noisy measurements of the network outputs.
In particular, we consider a nonlinear regression model with i.i.d.\ Gaussian observation noise:
\begin{equation}\label{RNB_eq1}
\mathbf{y}^{(p)} \;=\; \mathbf{f}(\mathbf{x}^{(p)}) \;+\; \boldsymbol{\varepsilon}^{(p)}, 
\qquad \boldsymbol{\varepsilon}^{(p)} \sim \mathcal{N}(\mathbf{0},\,\sigma_\epsilon^2 \mathbf I_K),\quad p=1,\dots,n,
\end{equation}
where, for $k=1,\dots,K$,
\begin{equation}\label{eq:shallow_bnn}
f_k(\mathbf x) \;=\; b_k \;+\; \sum_{j=1}^{H} v_{kj}\,
h\!\left(a_j \;+\; \sum_{i=1}^{I} u_{ji}\,x_i\right).
\end{equation}

Under the Bayesian approach, we place priors on the network parameters. Specifically,
\begin{equation}\label{RNB_eq3}
\{u_{ji}\}_{j,i}\stackrel{\text{i.i.d.}}{\sim}\mathcal{N}(0,\sigma_u^2),\;\;\;
\{a_j\}_{j=1}^H\stackrel{\text{i.i.d.}}{\sim}\mathcal{N}(0,\sigma_a^2),\;\;\;
\{v_{kj}\}_{k,j}\stackrel{\text{i.i.d.}}{\sim}\mathcal{N}\!\big(0,\tfrac{\sigma_v^2}{H}\big),\;\;\;
\{b_k\}_{k=1}^K\stackrel{\text{i.i.d.}}{\sim}\mathcal{N}(0,\sigma_b^2).
\end{equation}

These priors contribute to inference by regularizing the likelihood implied by the Gaussian nonlinear regression model.
In particular, the choice $v_{kj}\sim\mathcal{N}\!\big(0,\tfrac{\sigma_v^2}{H}\big)$ ensures that the output variance remains finite
as the hidden-layer width $H$ grows. Without the $1/H$ scaling, each additional hidden unit would add variance to the output, and
the variance of $f_k(\mathbf x)$ would typically diverge as $H\to\infty$.

\begin{architecture}[Bayesian feedforward neural network (BFNN)]
\label{def:bnn_prior_feedforward_simplified}
Let $\mathcal X\subseteq\mathbb R^{I}$ and $\mathcal Y\subseteq\mathbb R^{K}$.
Fix $H\in\mathbb N$ (number of hidden units) and a componentwise activation $h:\mathbb R\to\mathbb R$.
The observation model and the prior specification are given by \eqref{RNB_eq1}, \eqref{eq:shallow_bnn}, and \eqref{RNB_eq3}.
\end{architecture}

Under mild moment conditions (see \Cref{subsec:regularity_main}), finite-dimensional distributions of the network
outputs converge to Gaussian distributions, thus characterizing the limit as a GP.

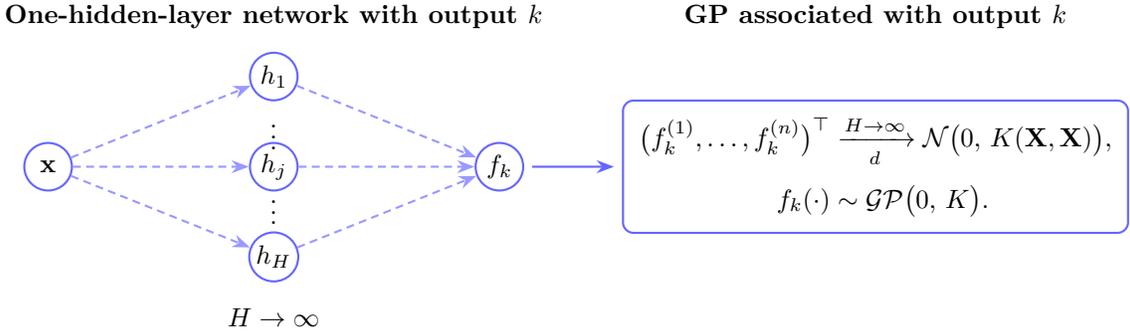
\begin{figure}[htbp]
\centering
\begin{adjustbox}{max width=\textwidth}
\begin{tikzpicture}[
  >=Stealth,
  every node/.style={font=\small},
  var/.style={circle,draw=blue!60,thick,minimum size=18pt,inner sep=1pt},
  arrow/.style={->,thick,blue!60},
  dottedarrow/.style={->,thick,blue!40,densely dashed}
]
\node[var] (x) at (0,0) {$\mathbf{x}$};

\node[var]  (h1) at (3,1.2) {$h_1$};
\node        (hv1) at (3,0.5) {$\vdots$};
\node[var]  (hj) at (3,0.0) {$h_j$};
\node        (hv2) at (3,-0.5) {$\vdots$};
\node[var]  (hH) at (3,-1.2) {$h_H$};

\node at (3,-2.0) {$H \to \infty$};

\node[var] (fk) at (6,0) {$f_k$};

\node at (3, 2.0) {\textbf{One-hidden-layer network with output $k$}};

\draw[dottedarrow] (x) -- (h1);
\draw[dottedarrow] (x) -- (hj);
\draw[dottedarrow] (x) -- (hH);

\draw[dottedarrow] (h1) -- (fk);
\draw[dottedarrow] (hj) -- (fk);
\draw[dottedarrow] (hH) -- (fk);

\node[draw=blue!60,thick,rounded corners,inner sep=6pt] (GP) at (11,0) {$
\begin{aligned}
\bigl(f_k^{(1)},\dots,f_k^{(n)}\bigr)^\top 
&\xrightarrow[d]{H\to\infty}
\mathcal{N}\bigl(0,\,K(\mathbf{X},\mathbf{X})\bigr),\\[0.3em]
f_k(\cdot) &\sim \mathcal{GP}\bigl(0,\,K\bigr).
\end{aligned}
$};

\node at (11,2.0) {\textbf{GP associated with output $k$}};

\draw[arrow,shorten >=3pt,shorten <=3pt]
  (fk.east) -- (GP.west);
\end{tikzpicture}
\end{adjustbox}
\caption{Limit of a wide one-hidden-layer network to a GP for output \(k\).
As \(H\to\infty\), the vector \(\bigl(f_k^{(1)},\dots,f_k^{(n)}\bigr)^\top\) converges in distribution
to a multivariate normal with covariance \(K(\mathbf{X},\mathbf{X})\), and the collection
\(\{f_k(\mathbf{x})\}_{\mathbf{x}}\) defines a GP \(\mathcal{GP}(0,K)\).}
\label{fig:nn_to_gp_output_k}
\end{figure}

\subsection{Limiting kernel and interpretation}
\label{subsec:limiting_kernel}

In the \(H\to\infty\) regime, the model is no longer described inferentially by individual weights.
Instead, it is fully characterized by a covariance function \(K(\mathbf{x},\mathbf{x}')\),
determined by \(h\) and the parameter priors. For \eqref{eq:shallow_bnn}, the limiting covariance takes the form
\begin{equation}
\label{eq:kernel_general_main}
K(\mathbf{x},\mathbf{x}') \;=\; \sigma_b^2 \;+\; \sigma_v^2 \,\mathbb{E}\!\left[h(Z)\,h(Z')\right],
\end{equation}
where \((Z,Z')\) denotes the pair of pre-activations associated with \((\mathbf{x},\mathbf{x}')\) induced by the priors on
\(a\) and \(u\) (details in Appendix~\ref{apx:preactivation_moments}).

\subsection{Minimal regularity condition}
\label{subsec:regularity_main}

To ensure that \eqref{eq:kernel_general_main} is well defined and that the output variance remains finite under
\(\mathbb{V}(v_{kj})=\sigma_v^2/H\), it suffices to assume that
\begin{equation}
\label{eq:second_moment_condition_main}
\mathbb{E}[h(Z)^2] < \infty \quad \text{for } Z \text{ distributed according to the implied prior on pre-activations.}
\end{equation}
Common activations satisfy \eqref{eq:second_moment_condition_main}: \(\tanh\) and sigmoid are bounded; ReLU and Leaky ReLU
grow at most linearly and thus have a finite second moment under Gaussian \(Z\).
A more general sufficient condition (polynomial growth) and a proof are given in Appendix~\ref{apx:regularity}.

\subsection{General convergence theorem}
\label{subsec:convergence_theorem_main}

We state a general version of the convergence result (in the sense of finite-dimensional distributions) for shallow BNNs.

\begin{theorem}[Convergence of a shallow BNN to a GP]
\label{thm:convergence_main}
Consider the model \eqref{eq:shallow_bnn} with the prior independent, centered parameters with finite variances and
the scaling \(\mathbb{V}(v_{kj})=\sigma_v^2/H\). Assume \(h\) and the priors are such that \(\mathrm{Var}(f_k(\mathbf{x}))<\infty\)
for all \(\mathbf{x}\) (e.g., \(\mathbb{E}[h(Z)^2]<\infty\)).
Then, for each \(k\in\{1,\dots,K\}\), as \(H\to\infty\), \(f_k\) converges in law (in the sense of finite-dimensional distributions)
to a Gaussian process
\[
f_k(\cdot)\sim \mathcal{GP}\!\bigl(0,\,K(\cdot,\cdot)\bigr),
\]
with kernel given by \eqref{eq:kernel_general_main}, namely
\begin{equation}
K(\mathbf{x},\mathbf{x}') = \sigma_b^2 + \sigma_v^2\,\mathbb{E}\!\left[h(Z)\,h(Z')\right],
\end{equation}
where \((Z,Z')\) is the pre-activation pair associated with \((\mathbf{x},\mathbf{x}')\) under the parameter priors.
\end{theorem}

\paragraph{Proof sketch.}
For fixed inputs \((\mathbf{x}_1,\dots,\mathbf{x}_n)\), the vector
\(\bigl(f_k(\mathbf{x}_1),\dots,f_k(\mathbf{x}_n)\bigr)\) can be written as a sum of i.i.d.\ terms over hidden units \(j\),
with variance controlled by the \(1/H\) scaling. Under \eqref{eq:second_moment_condition_main}, the multivariate CLT
(e.g.\ via Cramér--Wold) yields convergence to a multivariate normal with covariance \(K(\mathbf{X},\mathbf{X})\).
Since this holds for any finite set of inputs, the limit defines a GP. The full proof is given in Appendix~\ref{apx:proof_thm}.


\section{Mixed kernel derived from a BNN}
\label{sec:mixed_kernel_bnn}

In this section, we show how our \emph{mixed kernel} arises directly from a BNN prior.
The key point is that, under the standard scaling for wide hidden layers, the BNN prior induces a GP
whose covariance function can be written explicitly as an expectation under a bivariate Gaussian distribution.
Our mixed kernel is obtained by summing (and, optionally, adding) multiple kernels induced in this way, corresponding to
different components of the BNN prior.

\subsection{Kernels induced by nonlinear activations in BNNs}
\label{subsec:kernels_activations_bnn}

By the infinite-width convergence result (Theorem~\ref{thm:convergence_main}) under the architecture~\ref{eq:shallow_bnn},
for fixed inputs \(\mathbf{x},\mathbf{x}'\) the kernel induced by an activation \(h\) takes the form
\[
K_h(\mathbf{x},\mathbf{x}')
=
\sigma_b^2+\sigma_v^2\,\mathbb{E}\!\big[h(Z)\,h(Z')\big],
\]
where \((Z,Z')\) is the pair of pre-activations induced by the prior and is a centered bivariate Gaussian. Writing
\(\sigma_z^2=\mathrm{Var}(Z)\), \(\sigma_{z'}^2=\mathrm{Var}(Z')\), and \(\rho=\mathrm{Corr}(Z,Z')\), we obtain the integral representation
\[
\mathbb{E}\!\big[h(Z)\,h(Z')\big]
=
\iint_{\mathbb{R}^2} h(z)\,h(z')\;\phi_{\rho,\sigma_z,\sigma_{z'}}(z,z')\,dz\,dz',
\]
where \(\phi_{\rho,\sigma_z,\sigma_{z'}}\) denotes the bivariate Normal density with zero means, standard deviations \(\sigma_z,\sigma_{z'}\),
and correlation \(\rho\).

Next, we collect expressions available in the literature for the inner term
\(\mathbb{E}[h(Z)h(Z')]\), in closed form when available or via standard approximations,
for four reference activations:

\paragraph{tanh.}
The integral does not admit a simple closed form; one uses the standard approximation
\(\tanh(x)\approx \mathrm{erf}\!\left(\tfrac{\sqrt{\pi}}{2}x\right)\) \cite{williams1997computing}, which yields the arcsine kernel:

\begin{equation}\label{LCov1}
\mathbb{E}[\tanh(Z)\tanh(Z')]
\;\approx\;
\frac{2}{\pi}\,
\arcsin\!\Biggl(
\frac{\tfrac{\pi}{2}\,\rho\,\sigma_z \sigma_{z'}}
{\sqrt{\Bigl(1+\tfrac{\pi}{2}\sigma_z^2\Bigr)\Bigl(1+\tfrac{\pi}{2}\sigma_{z'}^2\Bigr)}}
\Biggr).
\end{equation}

\paragraph{Sigmoid.}
Similarly (via an $\mathrm{erf}$ approximation), one obtains:
\begin{equation}\label{LCov2}
\mathbb{E}[\mathrm{sig}(Z)\,\mathrm{sig}(Z')]
\;\approx\;
\frac{1}{4}
+
\frac{1}{2\pi}\,
\arcsin\!\Biggl(
\frac{\tfrac{\pi}{8}\,\rho\,\sigma_z \sigma_{z'}}
{\sqrt{\Bigl(1+\tfrac{\pi}{8}\sigma_z^2\Bigr)\Bigl(1+\tfrac{\pi}{8}\sigma_{z'}^2\Bigr)}}
\Biggr).
\end{equation}

\paragraph{ReLU.}
For \(\mathrm{ReLU}(z)=\max(0,z)\), a closed form is available:
\begin{equation}\label{LCov3}
\mathbb{E}[\mathrm{ReLU}(Z)\,\mathrm{ReLU}(Z')]
=
\frac{\sigma_z\,\sigma_{z'}}{2\pi}
\Bigl[
\sqrt{1-\rho^2} + \rho\bigl(\pi-\arccos(\rho)\bigr)
\Bigr].
\end{equation}

\paragraph{LeakyReLU.}
For \(\mathrm{LeakyReLU}_\alpha(z)=\max(z,\alpha z)\), define
\[
S(\rho):=\frac{1}{2\pi}\Bigl[\sqrt{1-\rho^2}+\rho\bigl(\pi-\arccos(\rho)\bigr)\Bigr],
\]
and obtain:
\begin{equation}\label{LCov4}
\mathbb{E}[\mathrm{LeakyReLU}_\alpha(Z)\,\mathrm{LeakyReLU}_\alpha(Z')]
=
\sigma_z\sigma_{z'}\Bigl[\alpha\,\rho+(1-\alpha)^2\,S(\rho)\Bigr].
\end{equation}

Finally, substituting each of the expressions above into
\(K_h(\mathbf{x},\mathbf{x}')=\sigma_b^2+\sigma_v^2\,\mathbb{E}[h(Z)h(Z')]\),
we obtain the corresponding induced kernels. The calculation steps and approximations are detailed Supplementary Material (Section S.2).

\subsection{Mixed kernel from the GP limit of a BNN}
\label{sec:practical-ident}

Each choice of activation function $h$ determines, in the infinite-width limit, a specific covariance function through the term $\mathbb{E}[h(Z)h(Z')]$. In applications, it is often useful to combine features of different activations. In this work, we implement such a combination through an additive (block) construction, i.e., as a sum of $M$ independent subnetworks, each using one activation $h_m$. This construction yields, in the GP limit, a kernel exactly in the form of an additive (and convex) combination of the component kernels, with no cross-terms.

For simplicity, we first consider a single-output version (the extension to each output component $k$ is immediate). Let
\[
z_j(\mathbf x)=a_j+\sum_{i=1}^{I}u_{ji}x_i,\qquad j=1,\dots,H,
\]
and define the function prior
\begin{equation}
\label{eq:block_mixture_prior}
f(\mathbf x)
=
b+\sum_{m=1}^M \sqrt{w_m}\sum_{j=1}^H v^{(m)}_{j}\,h_m\!\big(z_j(\mathbf x)\big),
\qquad
\sum_{m=1}^M w_m = 1,\;\; w_m\in(0,1),
\end{equation}
with
\[
v^{(m)}_{j}\overset{\text{i.i.d.}}{\sim}\mathcal N\!\left(0,\frac{\sigma_v^2}{H}\right),
\qquad
b\sim\mathcal N(0,\sigma_b^2),
\]
and independence across the collections $\{v^{(m)}_{j}\}_{j=1}^H$ for different $m$. This independence (together with $\mathbb E[v^{(m)}_{j}]=0$) implies that, when expanding the covariance, all cross-terms with $m\neq m'$ vanish, leading to an additive kernel in the GP limit. Moreover, to ensure that the CLT-based argument applies in the standard way, it suffices to assume that each component satisfies $\mathbb{E}[h_m(Z)^2]<\infty$ under the implied prior for the pre-activations; in that case, the output has finite variance and the GP limit is well defined.

\begin{proposition}[Convergence with mixed activations (additive blocks)]
\label{prop:convergencia_mista}
Consider the shallow network under the additive construction \eqref{eq:block_mixture_prior},
with the same assumptions as in Theorem~\ref{thm:convergence_main} for $(a_j,u_{ji})$, and with output-weight blocks
$\{v^{(m)}_{j}\}$ independent across components $m$, each scaled as $v^{(m)}_{j}\sim \mathcal{N}(0,\sigma_v^2/H)$.
Then, as $H\to\infty$, $f$ converges in distribution (in the sense of finite-dimensional distributions) to a Gaussian process
\[
f(\cdot)\sim \mathcal{GP}\!\bigl(0,K(\cdot,\cdot)\bigr),
\]
with covariance function
\begin{equation}
\label{eq:kernel_block_mixture}
K(\mathbf{x},\mathbf{x}')
=
\sigma_b^2+\sigma_v^2\sum_{m=1}^M w_m\,K_m(\mathbf{x},\mathbf{x}'),
\qquad
K_m(\mathbf{x},\mathbf{x}')
=
\mathbb{E}\!\left[h_m(Z)\,h_m(Z')\right],
\end{equation}
where $(Z,Z')$ denotes the pre-activation pair associated with $(\mathbf x,\mathbf x')$ under the priors (e.g., a centered bivariate Normal under Gaussian $(a_j,u_{ji})$), as described in Appendix~\ref{apx:preactivation_moments}.
\end{proposition}

\paragraph{Proof idea.}
Write $f(\mathbf x)=b+\sum_{m=1}^M \sqrt{w_m}\,g_m(\mathbf x)$, where
$g_m(\mathbf x)=\sum_{j=1}^H v^{(m)}_{j}\,h_m(z_j(\mathbf x))$.
Because the blocks $\{v^{(m)}_{j}\}$ are independent across $m$ and centered, cross-covariances vanish:
$\mathrm{Cov}(g_m(\mathbf x),g_{m'}(\mathbf x'))=0$ for $m\neq m'$.
For each fixed $m$, the standard wide-network argument (multivariate CLT with $1/H$ scaling) yields
$\mathrm{Cov}(g_m(\mathbf x),g_m(\mathbf x'))\to \sigma_v^2\,\mathbb E[h_m(Z)h_m(Z')]$.
Summing over $m$ and adding $\mathrm{Var}(b)=\sigma_b^2$ gives \eqref{eq:kernel_block_mixture}.
A complete proof is given in Appendix~\ref{apx:ident_mix_proof}.

\subsection{Identifiability of the mixed kernel}
\label{subsec:identificabilidade_kernel_misto}

The inner term $\mathbb{E}[h(Z)\,h(Z')]$ fully determines the dependence structure of the process over the input space and,
therefore, each choice of activation $h_m$ induces a specific covariance term
\[
K_m(\mathbf{x},\mathbf{x}')=\mathbb{E}\!\big[h_m(Z)\,h_m(Z')\big].
\]
In the infinite-width regime, the GP associated with the BNN has a mixture-form kernel
\begin{equation}
K(\mathbf{x},\mathbf{x}')
=\sigma_b^2+\sigma_v^2\sum_{m=1}^{M} w_m\,K_m(\mathbf{x},\mathbf{x}'),
\label{eq:kernel_mix_general}
\end{equation}
where we assume $w_m\ge 0$ (and, if desired for a mixture interpretation, $\sum_m w_m=1$).






\begin{table}[!ht]
  \centering
  \caption{Pearson correlation between the \emph{shapes} of the kernels (curves as functions of $\rho$ with $\sigma_z=\sigma_{z'}=1$),
  after centering and normalization.}
  \label{tab:corr-kernels}
  \begin{tabularx}{\textwidth}{@{}X c@{}}
    \toprule
    \textbf{Compared pairs} & \textbf{Correlation} \\
    \midrule
    tanh \(\times\) sigmoid & 0.9999 \\
    ReLU \(\times\) LeakyReLU (\(\alpha=0.1\)) & 0.9983 \\
    ReLU \(\times\) LeakyReLU (\(\alpha=0.3\)) & 0.9919 \\
    \bottomrule
  \end{tabularx}
\end{table}

The expressions \eqref{LCov1} to \eqref{LCov4} report four inner terms $K_m$ arising in the infinite-width limit.
Although these kernels are distinct in principle, in practice they can be very close as functions of
$\rho(\mathbf{x},\mathbf{x}')$ (for comparable scales of $\sigma_z$ and $\sigma_{z'}$).
As a consequence, different combinations of weights and hyperparameters may yield very similar covariance matrices
over a finite set of inputs, leading to (near) non-identifiability and ill-conditioned estimation problems,
a phenomenon that is typical in mixture models.

In particular, monotone saturating activations such as \textit{tanh} and sigmoid induce arc-sine family kernels
(via the \textit{erf} approximation) and produce smooth curves, monotone in $\rho$, with similar curvature around $\rho=0$
and saturation as $|\rho|\to 1$.
Likewise, ReLU and LeakyReLU induce nearly colinear shapes: the LeakyReLU term can be viewed as a perturbation of ReLU
by a linear term in $\rho$ weighted by $\alpha$, which for usual values of $\alpha$ leads to a numerically small difference.
This proximity is reflected in Table~\ref{tab:corr-kernels}, which reports very high Pearson correlations
between the curves (after centering and normalization) as functions of $\rho$.

In light of this, we adopt two representative components. For the smooth component, we take $K_{\mathrm{smooth}}\equiv K_1$,
since $K_1\approx K_2$ under the conditions above. For the angular component, we choose $K_{\mathrm{angular}}\equiv K_4$
(LeakyReLU) as the representative of the pair, because $K_3\approx K_4$ and $K_4$ includes the parameter $\alpha$ that controls asymmetry:
\[
\begin{aligned}
K_{\mathrm{smooth}}(\mathbf{x},\mathbf{x}') 
&\;\equiv\; K_{1}(\mathbf{x},\mathbf{x}')\;\approx\;K_{2}(\mathbf{x},\mathbf{x}'),\\
K_{\mathrm{angular}}(\mathbf{x},\mathbf{x}') 
&\;\equiv\; K_{4}(\mathbf{x},\mathbf{x}')\;\approx\;K_{3}(\mathbf{x},\mathbf{x}').
\end{aligned}
\]

Thus, we introduce a single mixing parameter $w\in(0,1)$ and define the reduced mixed kernel
\begin{align}
K_{\mathrm{mix}}(\mathbf{x},\mathbf{x}') 
&= \sigma_b^2 + \sigma_v^2\Bigl[
   w\,K_{\mathrm{smooth}}(\mathbf{x},\mathbf{x}')
   + (1-w)\,K_{\mathrm{angular}}(\mathbf{x},\mathbf{x}')
\Bigr]
\label{eq:k_mix_reduced}
\\
&= \sigma_b^2
   + \sigma_v^2\Biggl[
     w\,\frac{2}{\pi}\,\arcsin\!\Biggl(
       \frac{\tfrac{\pi}{2}\,\rho\,\sigma_z\,\sigma_{z'}}
            {\sqrt{\bigl(1+\tfrac{\pi}{2}\,\sigma_z^2\bigr)\,
                    \bigl(1+\tfrac{\pi}{2}\,\sigma_{z'}^2\bigr)}}
     \Biggr)
\nonumber\\
&+ (1-w)\,\Biggl(
       \alpha\,\rho\,\sigma_z\,\sigma_{z'}
       + (1-\alpha)^2\,
         \frac{\sigma_z\,\sigma_{z'}}{2\pi}\,
         \Bigl(\sqrt{1-\rho^2} + \rho\bigl(\pi-\arccos(\rho)\bigr)\Bigr)
     \Biggr)
   \Biggr].
\nonumber
\end{align}

The kernel $K_{\mathrm{mix}}$ is positive semidefinite because it is a convex combination (in $w$) of positive semidefinite kernels,
plus the bias-noise term $\sigma_b^2$; in particular, linear combinations with nonnegative coefficients preserve positive semidefiniteness.

\begin{proposition}[Identifiability of the mixed kernel]
\label{prop:ident_mix}

Consider the zero-mean GP with covariance $K_{\mathrm{mix}}$ in \eqref{eq:k_mix_reduced},
with parameter space
\[
\Theta = \bigl\{\,
\theta=(\sigma_a^2,\sigma_u^2,\sigma_b^2,\sigma_v^2,w,\alpha) \;:\;
\sigma_a^2,\sigma_u^2,\sigma_b^2,\sigma_v^2>0,\;
w\in(0,1),\;\alpha\in(0,1)
\,\bigr\}.
\]
Then the mapping
\[
\Theta \;\longrightarrow\; \mathcal{K}, 
\qquad \theta \mapsto K_{\mathrm{mix}}(\cdot,\cdot;\theta),
\]
is injective in the interior of $\Theta$. In particular, if for any $\theta_1,\theta_2\in\Theta$ we have
\[
K_{\mathrm{mix}}(\mathbf{x},\mathbf{x}';\theta_1)=K_{\mathrm{mix}}(\mathbf{x},\mathbf{x}';\theta_2),
\qquad \forall\,\mathbf{x},\mathbf{x}'\in\mathbb{R}^d,
\]
then necessarily $\theta_1=\theta_2$.
\end{proposition}

\paragraph{Proof idea.}
The complete proof is given in Appendix~\ref{apx:ident_mix_proof}.
The argument exploits families of pairs $(\mathbf x,\mathbf x')$ to separate parameters:
(i) by varying the angle between vectors of fixed norm, the constant term forces $\sigma_{b,1}^2=\sigma_{b,2}^2$;
(ii) using orthogonal pairs with $\|\mathbf x\|=\|\mathbf x'\|=r$ and varying $r$, one identifies $(\sigma_a^2,\sigma_u^2)$ via
the induced expressions for $\rho(r)$ and $\sigma_z^2(r)$; (iii) with these fixed, equality of the kernel implies equality
of the remaining parameters $(\sigma_v^2,w,\alpha)$, yielding $\theta_1=\theta_2$ in the interior of $\Theta$.
Boundary cases ($w\in\{0,1\}$ or $\alpha\in\{0,1\}$) are excluded since they may remove parameters from the kernel.

\subsection{Stratified input designs and practical identifiability}
\label{sec:stratification_ident}

Proposition~\ref{prop:ident_mix} establishes theoretical identifiability of
$K_{\mathrm{mix}}(\cdot,\cdot;\theta)$ in the interior of $\Theta$. In practice, however, identifiability can become
numerically fragile in moderate-to-high dimension when inputs $\{\mathbf x_i\}_{i=1}^n$ are sampled \emph{i.i.d.}.
The concentration of norms and angles around typical values, which is characteristic of high-dimensional spaces,
implies that the covariance matrix $\mathbf K(\theta)$ is evaluated only over a limited range of radial and angular variation
\cite{ledoux2001concentration,vershynin2018highdimprob}. As a result, the effectively observable information for
distinguishing kernel components is reduced, which can lead to numerically ill-conditioned estimation
\cite{beyer1999nearest,aggarwal2001surprising,rasmussen2006gaussian}.

Let $\mathbf K(\theta)\in\mathbb R^{n\times n}$ be the matrix with entries
$\mathbf K_{ij}(\theta)=K_{\mathrm{mix}}(\mathbf x_i,\mathbf x_j;\theta)$.
Information about $\theta$ is determined by two geometric aspects of the input set.
The first is the dispersion of the norms $\|\mathbf x_i\|$, which governs marginal variances,
and the second is the dispersion of the angles between pairs $(\mathbf x_i,\mathbf x_j)$, which governs covariances.
These two sources of variability feed, respectively, the diagonal and off-diagonal terms of $\mathbf K(\theta)$, as follows.

\begin{itemize}
\item Diagonal (radial information):
For the models considered, the pre-activation variance satisfies
$\sigma_z^2(\mathbf x)=\sigma_u^2\|\mathbf x\|^2+\sigma_a^2$.
Thus, separating $(\sigma_u^2,\sigma_a^2)$ depends on observing sufficient dispersion in $\|\mathbf x_i\|^2$
across inputs.

\item Off-diagonal (angular information):
For $i\neq j$, the \textit{kernel} depends on the pre-activation correlation
$\rho(\mathbf x_i,\mathbf x_j)$, which is a monotone function of the cosine similarity
$\tilde\rho(\mathbf x_i,\mathbf x_j)=\mathbf x_i^\top\mathbf x_j/(\|\mathbf x_i\|\|\mathbf x_j\|)$.
Therefore, distinguishing the smooth component $A(\rho)$ from the angular component $B(\rho;\alpha)$ requires
observing $\rho$ over a sufficiently wide range.
\end{itemize}

Under \emph{i.i.d.} sampling in high dimension, both $\|\mathbf x\|^2/I$ and $\tilde\rho(\mathbf x,\mathbf x')$
concentrate around typical values, with fluctuations that decay as $O(I^{-1/2})$
\cite{ledoux2001concentration,vershynin2018highdimprob}. In this regime, different values of $\theta$ may induce
very similar matrices $\mathbf K(\theta)$ in finite samples, which hampers numerical separation between kernel components
and can degrade estimation stability \cite{beyer1999nearest,aggarwal2001surprising,rasmussen2006gaussian}.

To make identifiability operational in practice, we compare two input designs,
a uniform (\emph{i.i.d.}) design and a stratified design.
The stratified design preserves variation along the two relevant axes of the \textit{kernel},
with radial stratification to increase dispersion in $\|\mathbf x\|$ and angular stratification to increase the diversity
of similarities between pairs of points. This choice allows us to study, in a controlled way, how the geometry of the input design
affects numerical identifiability and estimation stability in the simulation scenarios, isolating this effect from other factors.
This type of construction is consistent with recommendations for \emph{space-filling designs} to avoid clustering,
reduce voids, and increase coverage of the input domain \cite{santner2018design}.

It is important to emphasize that this stratification is not proposed as a strategy for real datasets,
where the inputs are already given and, in general, cannot be redesigned.
Here, it is used only in the simulations as a tool to investigate, in high dimension and in a controlled manner,
how the geometry of the input design affects numerical identifiability and estimation stability.
Supporting calculations and implementation details are provided in Supplementary Section S2.2.

\section{MAP estimation via Nyström and Bayesian prediction}
\label{sec:estimacao_map_nystrom}

In this section, we describe MAP estimation of the hyperparameters of the GP induced by the mixed kernel,
and we introduce the Nyström approximation to make the procedure scalable when $n$ is large.

\subsection{Optimization via MAP}
\label{subsec:map_opt}

Let $\mathcal{D}=\{(\mathbf x_i,y_i)\}_{i=1}^n$ and let $\theta$ denote the set of model hyperparameters.
MAP estimation is equivalent to minimizing the objective (\emph{log-posterior density})
\begin{equation}
\mathcal{L}_{\mathrm{MAP}}(\theta)
=
- \log p(\mathbf y \mid \mathbf X,\theta)\;-\;\log p(\theta),
\label{eq:map_obj}
\end{equation}
where the first term is the \emph{log-likelihood} induced by the GP measure and the second term is the contribution of the \emph{log-prior density}
(acting as probabilistic regularization). Updates are performed by gradient descent:
\[
\theta^{(t+1)}=\theta^{(t)}-\eta\,\nabla_\theta\,\mathcal{L}_{\mathrm{MAP}}(\theta).
\]
In practice, gradients are obtained by automatic differentiation applied to the numerical construction of $\widetilde{\mathbf K}_\theta$
(e.g., via a Cholesky factorization or low-rank approximations).

\subsection{GP log-likelihood}
\label{subsec:gp_mll}

We assume a regression model with independent Gaussian noise:
\[
y_i=f(\mathbf x_i)+\varepsilon_i,\qquad \varepsilon_i\sim\mathcal N(0,\sigma_\epsilon^2),
\]
and, under a zero-mean GP with kernel $K_{\mathrm{mix}}(\cdot,\cdot;\theta)$, we have
\[
\mathbf f\mid \mathbf X,\theta \sim \mathcal N(\mathbf 0,\mathbf K_\theta),
\qquad
\mathbf y\mid \mathbf X,\theta \sim \mathcal N(\mathbf 0,\widetilde{\mathbf K}_\theta),
\]
where
\begin{equation}
\widetilde{\mathbf K}_\theta=\mathbf K_\theta+\sigma_\epsilon^2\mathbf I,
\qquad
[\mathbf K_\theta]_{ij}=K_{\mathrm{mix}}(\mathbf x_i,\mathbf x_j;\theta).
\label{eq:Ktilde_def}
\end{equation}
Therefore, the log-likelihood takes the standard GP form:
\begin{equation}
\log p(\mathbf y\mid \mathbf X,\theta)
=
-\frac12\,\mathbf y^\top \widetilde{\mathbf K}_\theta^{-1}\mathbf y
-\frac12\,\log\det(\widetilde{\mathbf K}_\theta)
-\frac{n}{2}\log(2\pi).
\label{eq:gp_mll}
\end{equation}

\paragraph{Mixed-kernel form (reduced model).}
In our case,
\begin{equation}
\mathbf K_\theta
=
\sigma_b^2\,\mathbf 1\mathbf 1^\top
+
\sigma_v^2\Big[
w\,\mathbf K_{\mathrm{smooth}}(\sigma_a^2,\sigma_u^2)
+
(1-w)\,\mathbf K_{\mathrm{angular}}(\sigma_a^2,\sigma_u^2,\alpha)
\Big],
\label{eq:Ktheta_kmix_body}
\end{equation}
with $\theta=\{\sigma_\epsilon^2,\sigma_a^2,\sigma_u^2,\sigma_b^2,\sigma_v^2,\alpha,w\}$.

\subsection{MAP loss and gradients (standard form)}
\label{subsec:map_loss_grad}

With the priors
\[
\sigma_q^2\sim \mathrm{Inv\text{-}Gamma}(a_q,b_q),
\quad q\in\{\epsilon,a,u,b,v\},
\qquad
\alpha\sim\mathrm{Beta}(a_\alpha,b_\alpha),
\qquad
w\sim\mathrm{Beta}(a_w,b_w),
\]
the MAP loss (up to additive constants) is
\begin{align}
\mathcal L_{\mathrm{MAP}}(\theta)
&=
\frac12\,\mathbf y^\top \widetilde{\mathbf K}_\theta^{-1}\mathbf y
+\frac12\,\log\det(\widetilde{\mathbf K}_\theta)
+\sum_{q\in\{\epsilon,a,u,b,v\}}\Big[(a_q+1)\log\sigma_q^2 + \frac{b_q}{\sigma_q^2}\Big]
\notag\\
&\quad
-\,(a_\alpha-1)\log\alpha \;-\; (b_\alpha-1)\log(1-\alpha)
-\,(a_w-1)\log w \;-\; (b_w-1)\log(1-w).
\label{eq:map_loss_body}
\end{align}

For the likelihood term in \eqref{eq:gp_mll}, the gradient with respect to a generic parameter $\theta_j$ can be written compactly as
\begin{equation}
\frac{\partial}{\partial \theta_j}\Big(-\log p(\mathbf y\mid \mathbf X,\theta)\Big)
=
\frac12\,\mathrm{tr}\!\Big[
\big(\widetilde{\mathbf K}_\theta^{-1}\mathbf y\mathbf y^\top \widetilde{\mathbf K}_\theta^{-1}
-\widetilde{\mathbf K}_\theta^{-1}\big)\;
\frac{\partial \widetilde{\mathbf K}_\theta}{\partial \theta_j}
\Big].
\label{eq:grad_mll_standard}
\end{equation}
The specific derivatives $\partial \widetilde{\mathbf K}_\theta/\partial \theta_j$
(e.g., for $w$, $\alpha$, and the variance parameters) follow directly from
\eqref{eq:Ktilde_def}--\eqref{eq:Ktheta_kmix_body} and are listed in Supplementary Material (Section S.3) for reference.
From \eqref{eq:Ktheta_kmix_body},
\[
\frac{\partial \widetilde{\mathbf K}_\theta}{\partial w}
=
\sigma_v^2\big(\mathbf K_{\mathrm{smooth}}-\mathbf K_{\mathrm{angular}}\big),
\]
and the corresponding gradient follows from \eqref{eq:grad_mll_standard}.

\subsection{Nyström approximation}
\label{map_subsec:nystrom}

Let $\mathbf K \in \mathbb{R}^{n\times n}$ denote the covariance matrix induced by a kernel
$K_\theta(\cdot,\cdot)$. In this work, we use the approach proposed in \cite{banerjee2013efficient},
which we refer to as the \textit{Nyström} method, as it approximates $\mathbf K$ by a low-rank
decomposition constructed from a subset $S \subset \{1,\ldots,n\}$ with $|S|=r$ (the ``anchors'').

As shown in Eq.~\ref{eq:map_loss_body}, MAP optimization involves $\log\det(\widetilde{\mathbf K}_\theta)$
and products with $\widetilde{\mathbf K}_\theta^{-1}$. Therefore, we apply the Nyström approximation
directly to the noisy covariance matrix $\widetilde{\mathbf K}_\theta$:

\begin{equation}
\label{eq:nystrom_noise}
\widetilde{\mathbf K}_\theta^{(r)}
=
\widetilde{\mathbf K}_{\theta,:S}\,
\widetilde{\mathbf K}_{\theta,SS}^{-1}\,
\widetilde{\mathbf K}_{\theta,S:},
\qquad
\widetilde{\mathbf K}_\theta=\mathbf K_\theta+\sigma_\epsilon^2\mathbf I,
\qquad
S \;\text{with}\; |S|=r.
\end{equation}
where $S\subset\{1,\dots,n\}$ is a subset of size $r$.
With this construction, the dominant computational cost (e.g., forming the required terms and solving
the associated linear systems) typically reduces to $\mathcal{O}(n r^2)$ instead of $\mathcal{O}(n^3)$. Anchors $S$ can be selected by different strategies, such as random sampling, \textit{k-means++}
initialization (or centroids obtained via \textit{k}-means) \cite{arthur2007kmeans},
or randomized sampling/projection techniques that approximate the column space of $\mathbf K$
\cite{rahimi2007random,halko2011finding}. In this work, we adopt two criteria depending on the
experimental context:
\begin{itemize}
    \item \textit{first}, where $S=\{1,\ldots,r\}$ (the first $r$ points), used in simulations (via Vecchia approximation) due to its
simplicity and reproducibility; or
    \item \textit{k-means++}, which selects the first centroid uniformly at random from the data and then
selects each subsequent centroid with probability proportional to the squared distance to the nearest
already-chosen centroid \cite{arthur2007kmeans}.
\end{itemize}

In summary, the Nyström approximation avoids the exact factorization of an $n\times n$ matrix---the main
source of the $\mathcal{O}(n^3)$ cost in GPs---by replacing it with low-rank operations. This makes
it practical to compute $\log\det(\widetilde{\mathbf K}_\theta)$ and solve $\widetilde{\mathbf K}_\theta
\mathbf u=\mathbf y$ within MAP optimization on large-scale datasets.

\subsection{Bayesian prediction}


In a fully Bayesian analysis, uncertainty about the model parameters $\theta$ is propagated to prediction through the posterior predictive distribution:
\[
p(y_\star \mid \mathcal{D}, \mathbf{x}_\star)
=
\int p(y_\star \mid \mathbf{x}_\star, \theta)\, p(\theta \mid \mathcal{D}) \, d\theta.
\]
Thus, prediction is not obtained by conditioning on a single point estimate of $\theta$, but by averaging the conditional predictions $p(y_\star \mid \mathbf{x}_\star, \theta)$ weighted by the posterior distribution $p(\theta \mid \mathcal{D})$. As a result, predictive intervals reflect both the intrinsic variability of the model and uncertainty about $\theta$. Although conceptually appealing, this approach is typically infeasible in large-scale problems, since evaluating the posterior predictive distribution requires repeatedly recomputing quantities based on the covariance matrix $\widetilde{\mathbf K}_\theta$, including costly terms involving $\widetilde{\mathbf K}_\theta^{-1}$.

To balance uncertainty quantification and computational cost, we estimate the parameters via MAP, obtaining a point estimate $\hat{\theta}$, and perform prediction conditional on $\hat{\theta}$. Although this does not explicitly integrate over parameter uncertainty, predictions remain probabilistic because the GP provides a full predictive distribution for $y_\star$. Accordingly, let $\mathbf{X}\in\mathbb{R}^{n\times I}$ denote the training inputs and $\mathbf{y}\in\mathbb{R}^{n}$ the observed responses. We define
\[
\widetilde{\mathbf{K}}_{\hat{\theta}}
=
\mathbf{K}_{\hat{\theta}}+\hat{\sigma}_\epsilon^2\,\mathbf{I},
\qquad
[\mathbf{K}_{\hat{\theta}}]_{ij}
=
K_{\mathrm{mix}}(\mathbf{x}_i,\mathbf{x}_j;\hat{\theta}).
\]
For a new input $\mathbf{x}_\star$, we let
\[
\mathbf{k}_\star =
\big(
K_{\mathrm{mix}}(\mathbf{x}_\star,\mathbf{x}_1;\hat{\theta}),
\ldots,
K_{\mathrm{mix}}(\mathbf{x}_\star,\mathbf{x}_n;\hat{\theta})
\big)^\top,
\qquad
k_{\star\star}
=
K_{\mathrm{mix}}(\mathbf{x}_\star,\mathbf{x}_\star;\hat{\theta}).
\]
Then, the posterior predictive distribution is Gaussian:
\begin{equation}
\label{eq:gp_pred_pt}
y_\star \mid \mathbf{y},\mathbf{X},\mathbf{x}_\star,\hat{\theta}
\sim
\mathcal{N}\!\big(\mu_\star,\sigma_\star^2\big),
\end{equation}
with mean and variance given by
\begin{equation}
\label{eq:gp_pred_mom_pt}
\mu_\star
=
\mathbf{k}_\star^\top
\widetilde{\mathbf{K}}_{\hat{\theta}}^{-1}
\mathbf{y},
\qquad
\sigma_\star^2
=
k_{\star\star}
-
\mathbf{k}_\star^\top
\widetilde{\mathbf{K}}_{\hat{\theta}}^{-1}
\mathbf{k}_\star
+
\hat{\sigma}_\epsilon^2.
\end{equation}

Note that, given $\hat{\theta}$, GP prediction requires manipulating the covariance matrix $\widetilde{\mathbf K}_{\hat{\theta}}$. Computing and manipulating this matrix exactly is impractical at large scale. Hence, although the predictive expressions are written in terms of $\widetilde{\mathbf K}_{\hat{\theta}}^{-1}$, in practice we replace $\widetilde{\mathbf K}_{\hat{\theta}}$ by its Nyström approximation $\widetilde{\mathbf K}_{\hat{\theta}}^{(r)}$ (Eq.~\ref{eq:nystrom_noise}) and compute the products $\big(\widetilde{\mathbf K}_{\hat{\theta}}^{(r)}\big)^{-1}\mathbf y$ and $\big(\widetilde{\mathbf K}_{\hat{\theta}}^{(r)}\big)^{-1}\mathbf k_\star$ via low-rank algebra (equivalently, by solving linear systems in the approximated space), avoiding explicit matrix inversion and enabling scalable evaluation of predictive means and variances \cite{banerjee2013efficient}.

The proposed GP model, conditioned on $\hat{\theta}$ and using the Nyström approximation, offers important advantages over finite NNs. First, whereas in finite networks the number of parameters grows with the number of neurons $H$, the GP limit is governed by a small set of parameters, such as the variances $(\sigma_a^2,\sigma_u^2,\sigma_b^2,\sigma_v^2,\sigma_\epsilon^2)$ and, when applicable, parameters indexing the activation function, which in turn determine the covariance function. Second, uncertainty quantification is more faithful and robust, especially outside the data domain. In deterministic NNs, prediction is purely pointwise and there is no predictive distribution. Moreover, in finite NNs with a nugget term, predictive variance tends to be approximately constant and dominated by the noise term, remaining small even when $\mathbf{x}_\star$ is far from the training data. In contrast, in the proposed GP, predictive variance includes an epistemic component that naturally increases as $\mathbf{x}_\star$ becomes less correlated with the training inputs, approaching the prior variance in regions poorly informed by data. Thus, even with parameters fixed via MAP, the GP consistently captures both aleatoric and epistemic uncertainty.

Finally, the predictive mean $\mu_\star$ and predictive variance $\sigma_\star^2$ (Eqs.~\ref{eq:gp_pred_pt}--\ref{eq:gp_pred_mom_pt}) provide the basis for computing the evaluation metrics: the Mean absolute error (MAE) and the Root mean squared error (RMSE) assess point-prediction accuracy using $\mu_\star$, whereas the Mean expected squared error (MESE) and the Standard deviation of the expected squared error (SDESE) jointly assess accuracy and the quality of uncertainty quantification by explicitly incorporating $\sigma_\star^2$ (see Section \ref {sec:training_evaluation_protocol} for a formal definition of those metrics). Under the Gaussian assumption, the expected squared error decomposes as
\[
\mathbb{E}\!\left[(Y_\star-y_\star)^2 \mid \mathcal{D},\mathbf{x}_\star,\hat{\theta}\right]
=
(\mu_\star-y_\star)^2+\sigma_\star^2,
\]
i.e., the sum of the predictive-mean error and the predictive uncertainty, reflecting both accuracy and the reliability of the reported uncertainty.

\section{Experimental Setup}
\label{sec:experimentos}

We evaluate the proposed method on both simulated data and two public regression datasets. In the simulations, we vary the sample size \(n\), the input dimension \(I\), and the sampling design for \(X\) (uniform vs.\ stratified), and we compare anchor-selection strategies for the Nyström approximation. On real data, we report predictive performance and total runtime on \textit{Superconductivity} and \textit{YearPredictionMSD}, including a scalability study on the full \textit{YearPredictionMSD} dataset by varying the Nyström rank \(r\).

\subsection{Simulated data: generation, scenarios, and replications}
\label{sec:dados_simulados}

We consider eight experimental scenarios. In each scenario, inputs $\mathbf{x}_i \in [0,1]^I$ are generated under two sampling designs:
(i) uniform and (ii) stratified over the unit hypercube. We then apply a centering transformation,
\( \tilde{\mathbf{x}}_i=\mathbf{x}_i-0.5, \)
ensuring $\tilde{\mathbf{x}}_i\in[-0.5,0.5]^I$, consistent with the model parameterization. The latent function is defined as a Gaussian process
$f\sim\mathcal{GP}(0,K_{\mathrm{mix}})$. Observations are obtained as
\[
y(\tilde{\mathbf{x}})=f(\tilde{\mathbf{x}})+\varepsilon(\tilde{\mathbf{x}}),
\qquad \varepsilon(\tilde{\mathbf{x}})\stackrel{iid}{\sim}\mathcal N(0,\sigma_\varepsilon^2),
\]
with $\varepsilon$ assumed independent of $f$. The \emph{nugget} $\varepsilon$ is calibrated per scenario as a fixed fraction $\eta$ of the average
marginal variance of the kernel evaluated at the input locations:
\begin{equation}
\sigma_\varepsilon^2=\eta\,\overline K,
\qquad
\overline K=\frac{1}{n}\sum_{i=1}^n K_{\mathrm{mix}}(\tilde{\mathbf x}_i,\tilde{\mathbf x}_i),
\label{eq:nugget_frac}
\end{equation}
with $\eta=0.04$. In the implementation, $\overline K$ is computed in \emph{batches} without forming the full $n\times n$ matrix.
To enable scalable GP sampling in regimes with large \(n\) without incurring cubic cost, we generate \(f\) using a sequential
\textit{Vecchia}-type construction: we sample an initial exact block of size $N_{\text{init}}=500$ via Cholesky, and for each new point,
we sample the univariate conditional distribution based on the $N_{\mathrm{viz}}=500$ nearest neighbors among the previously simulated ones. This procedure is used
\emph{only} in the simulation stage (data generation). We vary $n\in\{10{.}000,20{.}000,50{.}000\}$, $I\in\{20,80\}$ and, for $n=50{.}000$,
the design of $X$ (uniform vs.\ stratified), yielding eight scenarios (C1--C8) summarized in Table~\ref{tab:cenarios_simulacao_kmix}.
To quantify variability induced by the randomness of the latent process and observation noise, we perform $R=20$ replications for scenario C1,
keeping the inputs fixed and varying only \texttt{seed\_y}. A summary of the replications is reported in the Supplementary Material (Table S.1),
and full details are provided in the Supplementary Material (Section S.5).

\begin{table}[!ht]
\centering
\caption{Simulation scenarios using a \textit{Vecchia}-type construction with a nugget automatically calibrated from $\overline{K_{\mathrm{mix}}(x,x)}$.}
\label{tab:cenarios_simulacao_kmix}
\renewcommand{\arraystretch}{1.15}
\resizebox{\linewidth}{!}{%
\begin{tabular}{l l r r r r r r}
\toprule
\textbf{Scenario} & \textbf{Design of $X$} & \textbf{$I$} & \textbf{$n$} &
$\boldsymbol{\overline{K_{\mathrm{mix}}(x,x)}}$ &
\textbf{Nugget $(\sigma_\varepsilon)$} &
$\boldsymbol{\sigma_\varepsilon^2}$ &
\textbf{Time (s)} \\
\midrule
C1 & Uniform      & 20 & 10000 & 2.132745 & 0.292078 & 0.085310 & 63.72 \\
C2 & Uniform      & 80 & 10000 & 3.769967 & 0.388328 & 0.150799 & 34.90 \\
C3 & Uniform      & 20 & 20000 & 2.131981 & 0.292026 & 0.085279 & 68.46 \\
C4 & Uniform      & 80 & 20000 & 3.769044 & 0.388281 & 0.150762 & 74.07 \\
C5 & Uniform      & 20 & 50000 & 2.131707 & 0.292007 & 0.085268 & 182.54 \\
C6 & Uniform      & 80 & 50000 & 3.768044 & 0.388229 & 0.150722 & 215.66 \\
C7 & Stratified   & 20 & 50000 & 2.132325 & 0.292050 & 0.085293 & 186.13 \\
C8 & Stratified   & 80 & 50000 & 3.772796 & 0.388474 & 0.150912 & 227.50 \\
\bottomrule
\end{tabular}%
}
\end{table}

Table~\ref{tab:cenarios_simulacao_kmix} defines the eight simulation regimes by varying \(n\), the input dimension \(I\), and the design of \(\mathbf{X}\).
Automatically calibrating the \emph{nugget} as a fixed fraction of the average marginal variance of \(K_{\mathrm{mix}}(\mathbf{x},\mathbf{x})\)
standardizes the noise scale across scenarios, avoiding ad hoc choices when \(K_{\mathrm{mix}}\) changes with \(I\) or with the sampling design.

\subsection{Real data and preprocessing}
\label{sec:real_data}

We evaluate the method on the public regression benchmarks \textit{Superconductivity} and \textit{YearPredictionMSD}.
For \textit{YearPredictionMSD}, we analyze random subsamples with $n\in\{10{,}000,20{,}000,50{,}000\}$ and, additionally,
the full dataset while varying the Nyström rank $r$, in order to characterize the cost--accuracy trade-off of the approximation.

Let $X\in\mathbb{R}^{n\times I}$ denote the covariate matrix and $y\in\mathbb{R}^n$ the response vector.
We rescale each column of $X$ to $[0,1]^I$ using min--max normalization (fit on the training set and applied to the test set),
and we center the inputs via $\tilde X = X - 0.5$.
We standardize the target on the training set via $\tilde y=(y-\mu_y)/\sigma_y$ and de-standardize predictions to report all metrics on the original scale.
All datasets have no missing values; hence no imputation was required.

\subsection{Training and evaluation protocol}
\label{sec:training_evaluation_protocol}

To ensure direct comparability across scenarios, we adopt a fixed-budget training protocol: we train the model for $50$ epochs in
all configurations, keeping all other optimization choices unchanged.
Model fitting is performed via MAP estimation (Eq.~\ref{eq:map_loss_body}), by minimizing
\[
\mathcal L_{MAP}^{GP}(\theta)=\mathrm{NLL}(\theta)-\log p(\theta),
\qquad
\mathrm{NLL}(\theta)=-\log p(\mathbf y\mid \mathbf X,\theta),
\]
where $p(\theta)$ denotes the prior.

For scalability, we use the \textit{Nyström} approximation applied to the mixed kernel $K_{\text{mix}}$.
In the simulation scenarios, we fix the rank at $r=500$.
For real data, we use $r=500$ on the full \textit{Superconductivity} dataset and on the \textit{YearPredictionMSD} subsamples,
so as to keep computational cost comparable across configurations.
Only in the experiment on the full \textit{YearPredictionMSD} dataset do we vary $r$, to characterize the cost--accuracy trade-off.
Anchor selection is performed either by \textit{first} or by \textit{$k$-means}, depending on the experiment.

In all runs, we optimize the parameters with Adam (learning rate $10^{-3}$), computing objective terms with mini-batches. The prior term is always included in the objective.
The nugget is estimated jointly with the remaining parameters, with a dedicated learning rate of $10^{-3}$,
and is automatically initialized using the rule in Eq.~\ref{eq:nugget_frac} (with $\eta=0.04$).
When available, kernel parameters are initialized with the values used in the simulation.

Across all datasets, we use reproducible splits (shared random seeds) and reserve $10\%$ for testing. After training on $\mathcal{D}_{\text{train}}$, we evaluate predictive performance on the held-out test set. For each test input $\mathbf{x}_i$, the fitted GP yields a Gaussian predictive distribution
$Y_i \mid \mathcal{D}_{\text{train}},\mathbf{x}_i \sim \mathcal{N}(\hat y_i,\hat\sigma_i^2)$.
Let $(\hat y_i,\hat\sigma_i^2)$ denote the predictive mean and variance for each test point $(\mathbf{x}_i,y_i)$, $i=1,\dots,n_{\text{test}}$.
Then, we compute point-accuracy metrics,
\[
\mathrm{MAE}=\frac{1}{n_{\text{test}}}\sum_{i=1}^{n_{\text{test}}}\lvert \hat y_i-y_i\rvert,
\qquad
\mathrm{MSE}=\frac{1}{n_{\text{test}}}\sum_{i=1}^{n_{\text{test}}}(\hat y_i-y_i)^2,
\qquad
\mathrm{RMSE}=\sqrt{\mathrm{MSE}}.
\]
To jointly assess accuracy and uncertainty, we use the Expected squared error (ESE) under the predictive distribution,
\[
\mathrm{ESE}_i
=
\mathbb{E}\!\left[(Y_i-y_i)^2 \mid \mathcal{D}_{\text{train}},\mathbf{x}_i\right]
=
(\hat y_i-y_i)^2+\hat\sigma_i^2,
\]
and report
\[
\mathrm{MESE}=\frac{1}{n_{\text{test}}}\sum_{i=1}^{n_{\text{test}}}\mathrm{ESE}_i,
\qquad
\mathrm{SDESE}=\sqrt{\frac{1}{n_{\text{test}}-1}\sum_{i=1}^{n_{\text{test}}}\big(\mathrm{ESE}_i-\mathrm{MESE}\big)^2}.
\]
For real datasets, in addition to metrics on the original target scale, we also report normalized versions based on the training-target scale to enable comparisons across datasets with different units (Supplementary Material S5.4).

Overall, the experimental design and protocol above define a reproducible procedure to compare, in a controlled manner, the effect of the simulation regime, the \textit{Nyström} anchor-selection strategy, and, in the large-scale setting, the rank $r$.
In the next section, we report MAP estimates, test-set predictive metrics, and total time, highlighting the cost--accuracy trade-off induced by the scalable approximation.

\subsection{Results}

\subsubsection{Simulation results}

The Tables~\ref{tab:map_params_50eps_anchors} and \ref{tab:test_metrics_50eps_anchors} disentangle two effects:
(i) the stability of the MAP estimates of the \textit{kernel}, and
(ii) the impact of the anchor-selection strategy on the quality of the Nyström approximation.
Across all scenarios, the estimated structural parameters \((\hat\sigma_b^2,\hat\sigma_v^2,\hat\sigma_u^2,\hat\sigma_a^2,\hat\alpha,\hat w)\)
remain very similar under first and \(k\)-means, suggesting stable optimization under a fixed training budget,
with no evidence of collapse to degenerate cases (e.g., \(w\in\{0,1\}\)).

In contrast, the test-set metrics highlight a clear cost--accuracy trade-off. In more challenging settings (notably \(I=80\)),
\(k\)-means anchor selection tends to yield more noticeable reductions in MAE/RMSE, whereas in lower dimension the differences are small.
This pattern is consistent with the role of anchors in Nyström, since better coverage of the input space improves the approximation of the correlated
component of the process. As expected, these gains may come at increased computational cost due to clustering, which is reflected in the total runtime.

\begin{table}[!ht]
\centering
\caption{MAP estimates under two anchor-selection strategies in Nyström.}
\label{tab:map_params_50eps_anchors}
\renewcommand{\arraystretch}{0.60} 
\setlength{\tabcolsep}{10pt}        
\resizebox{\textwidth}{!}{%
\begin{tabular}{llrrrrrrr}
\toprule
Scenario & Anchors & $\hat{\sigma}_b^2$ & $\hat{\sigma}_v^2$ & $\hat{\sigma}_u^2$ & $\hat{\sigma}_a^2$ & $\hat{\alpha}$ & $\hat{w}$ & $\hat{\sigma}_\epsilon^2$ \\
\midrule
C1 & first   & 1.025371 & 0.980080 & 0.981301 & 0.994972 & 0.508588 & 0.495693 & $8.56\times10^{-2}$ \\
C1 & $k$-means & 1.026216 & 0.987962 & 0.995776 & 0.987359 & 0.508575 & 0.496666 & $8.29\times10^{-2}$ \\
\midrule
C2 & first   & 1.015887 & 1.031211 & 1.030972 & 0.968986 & 0.487482 & 0.512434 & $1.64\times10^{-1}$ \\
C2 & $k$-means & 1.028323 & 1.032194 & 1.032033 & 0.968549 & 0.487544 & 0.511323 & $1.63\times10^{-1}$ \\
\midrule
C3 & first   & 1.028055 & 0.971542 & 0.972107 & 0.983965 & 0.511472 & 0.491441 & $8.64\times10^{-2}$ \\
C3 & $k$-means & 1.023113 & 0.974400 & 0.976894 & 0.996529 & 0.510373 & 0.490998 & $8.18\times10^{-2}$ \\
\midrule
C4 & first   & 0.999669 & 1.031572 & 1.031348 & 0.968746 & 0.487415 & 0.511870 & $1.64\times10^{-1}$ \\
C4 & $k$-means & 1.018052 & 1.031881 & 1.031831 & 0.968938 & 0.487784 & 0.509282 & $1.63\times10^{-1}$ \\
\midrule
C5 & first   & 1.028704 & 0.971421 & 0.971586 & 0.994188 & 0.511165 & 0.490527 & $8.34\times10^{-2}$ \\
C5 & $k$-means & 1.025100 & 0.969336 & 0.969129 & 1.001332 & 0.511568 & 0.489583 & $8.21\times10^{-2}$ \\
\midrule
C6 & first   & 1.001530 & 1.031871 & 1.031808 & 0.969091 & 0.487441 & 0.512053 & $1.64\times10^{-1}$ \\
C6 & $k$-means & 1.014398 & 1.031409 & 1.031123 & 0.969269 & 0.487602 & 0.511393 & $1.64\times10^{-1}$ \\
\midrule
C7 & first   & 1.025906 & 0.971412 & 0.971880 & 1.002389 & 0.511691 & 0.489485 & $8.32\times10^{-2}$ \\
C7 & $k$-means & 1.024638 & 0.971043 & 0.972473 & 1.013075 & 0.512253 & 0.489012 & $8.33\times10^{-2}$ \\
\midrule
C8 & first   & 0.983979 & 1.031795 & 1.031675 & 0.968993 & 0.487443 & 0.512017 & $1.64\times10^{-1}$ \\
C8 & $k$-means & 1.018084 & 1.031796 & 1.031743 & 0.969018 & 0.487622 & 0.511714 & $1.64\times10^{-1}$ \\
\bottomrule
\end{tabular}%
}
\end{table}

\begin{table}[!ht]
\centering
\caption{Total time and test-set metrics under two Nyström anchor-selection strategies.}
\label{tab:test_metrics_50eps_anchors}
\renewcommand{\arraystretch}{0.60} 
\setlength{\tabcolsep}{10pt} 
\resizebox{\textwidth}{!}{%
\begin{tabular}{llrrrrrr}
\toprule
Scenario & Anchors & Total time (s) & MAE & MSE & RMSE & MESE & SDESE \\
\midrule
C1 & first    & 1.334 & 0.265983 & 0.112261 & 0.335053 & 0.236112 & 0.165608 \\
C1 & $k$-means  & 2.155 & 0.263822 & 0.110102 & 0.331816 & 0.233065 & 0.161552 \\
\midrule
C2 & first    & 1.284 & 0.573294 & 0.504148 & 0.710034 & 0.892681 & 0.665183 \\
C2 & $k$-means  & 2.514 & 0.534228 & 0.440660 & 0.663822 & 0.829185 & 0.585415 \\
\midrule
C3 & first    & 1.881 & 0.265764 & 0.110175 & 0.331926 & 0.233302 & 0.152864 \\
C3 & $k$-means  & 2.137 & 0.264066 & 0.109563 & 0.331004 & 0.229287 & 0.153301 \\
\midrule
C4 & first    & 2.101 & 0.571924 & 0.515084 & 0.717694 & 0.903406 & 0.718319 \\
C4 & $k$-means  & 2.334 & 0.562438 & 0.498106 & 0.705766 & 0.886464 & 0.707048 \\
\midrule
C5 & first    & 1.623 & 0.269528 & 0.114101 & 0.337789 & 0.233814 & 0.160743 \\
C5 & $k$-means  & 2.408 & 0.267161 & 0.111680 & 0.334185 & 0.231012 & 0.159149 \\
\midrule
C6 & first    & 1.550 & 0.593175 & 0.562680 & 0.750120 & 0.950702 & 0.807495 \\
C6 & $k$-means  & 3.511 & 0.566497 & 0.503485 & 0.709567 & 0.892490 & 0.718490 \\
\midrule
C7 & first    & 1.389 & 0.270088 & 0.115020 & 0.339146 & 0.234563 & 0.165507 \\
C7 & $k$-means  & 2.592 & 0.270696 & 0.116113 & 0.340753 & 0.237124 & 0.168015 \\
\midrule
C8 & first    & 1.335 & 0.583224 & 0.537355 & 0.733045 & 0.927608 & 0.777588 \\
C8 & $k$-means  & 1.395 & 0.581387 & 0.538477 & 0.733810 & 0.929836 & 0.774792 \\
\bottomrule
\end{tabular}%
}
\end{table}

To isolate the effect of the input design, we compare pairs with fixed $n$ and $I$ that differ only in the sampling design of $X$.
In low dimension ($I=20$), the design effect is small, with C5 and C7 yielding very similar metrics.
In high dimension ($I=80$), the effect depends on the anchor-selection strategy.
With first, the stratified design (C8) slightly improves over the uniform design (C6), whereas with \(k\)-means the opposite occurs, with C8 performing worse than C6.
Overall, the stratified design does not yield systematic gains its effect is secondary at $I=20$ and, at $I=80$, it interacts with the choice of anchors.
This pattern is consistent with the idea that, in more complex regimes, input-space coverage and the anchor-selection mechanism can interact to affect approximation quality.

\begin{table}[!ht]
\centering
\caption{Summary of the 20 replications of scenario C1 ($r=500$, \textit{$k$-means} anchors, $n_{\text{train}}=9000$).}
\label{tab:C1_reps_summary_50eps}
\renewcommand{\arraystretch}{1.10}
\footnotesize
\begin{tabularx}{\textwidth}{l >{\centering\arraybackslash}X >{\centering\arraybackslash}X}
\toprule
\textbf{Quantidade} & \textbf{Média (dp)} & \textbf{[min, max]} \\
\midrule
MAE        & 0.273514 (0.004816) & [0.263822, 0.282890] \\
RMSE       & 0.342126 (0.005893) & [0.331816, 0.353658] \\
MESE       & 0.239397 (0.004587) & [0.231336, 0.246951] \\
SDESE      & 0.167666 (0.005495) & [0.161552, 0.179720] \\
total\_s   & 2.486 (0.616)        & [1.829, 4.463] \\
$\sigma_{\varepsilon}^2$ & 0.08341 (0.00142) & [0.08033, 0.08700] \\
w          & 0.499987 (0.002367) & [0.494409, 0.504269] \\
alpha      & 0.509978 (0.001317) & [0.508133, 0.513169] \\
\(\sigma_b^2\)   & 1.028096 (0.001465) & [1.025079, 1.031158] \\
\(\sigma_v^2\)   & 0.975803 (0.006260) & [0.963308, 0.987962] \\
\(\sigma_u^2\)   & 0.982022 (0.006519) & [0.969905, 0.995776] \\
\(\sigma_a^2\)   & 0.980596 (0.007583) & [0.965003, 0.999919] \\
\bottomrule
\end{tabularx}
\end{table}

The results (Table~\ref{tab:C1_reps_summary_50eps}) show that test-set metrics exhibit non-negligible variability across replications (e.g., mean RMSE $\approx 0.342$ with sd $\approx 0.006$), while the MAP estimates of the parameters remain stable (e.g., $w \approx 0.5$ with low dispersion).
This suggests that Nyström-based fitting is numerically stable in scenario C1 and reinforces the importance of reporting performance as mean$\pm$sd, avoiding conclusions drawn from a single realization.
To contextualize the magnitude of this variability, note that the replications were conducted using \textit{$k$-means} anchors. The comparison between anchor-selection strategies (\textit{first} vs.\ \textit{$k$-means}) is presented separately for scenario C1 (Tables~\ref{tab:map_params_50eps_anchors}--\ref{tab:test_metrics_50eps_anchors}), where \textit{$k$-means} typically yields a slight error reduction at the cost of increased runtime due to the additional clustering step. Taken together, the replications indicate that the MAP estimates remain stable, while predictive performance varies in a manner consistent with the randomness of the latent process and observation noise.

Next, we evaluate the method on real datasets and examine its behavior in the large-scale regime.

\subsubsection{Real data}

To maintain comparability with the simulation scenarios and control computational cost, we use \textit{$k$-means} anchor selection for both the \textit{Superconductivity} dataset and the \textit{YearPredictionMSD} subsamples. For the full \textit{YearPredictionMSD} dataset, we instead adopt the \textit{first} strategy, aiming for improved scalability and lower preprocessing overhead.

The results in Tables~\ref{tab:real_map_params_50eps} and~\ref{tab:real_metrics_50eps} show that the MAP estimates are stable across the \textit{YearPredictionMSD} settings, in particular with $\hat{\alpha}\approx 0.488$ and $\hat{w}\approx 0.512$, indicating that the fit preserves a balanced mixture between the components of $K_{\text{mix}}$ and does not collapse to degenerate cases ($w\in\{0,1\}$). Since training is carried out on the standardized target scale, the \emph{nugget} is estimated in that scale; after de-standardization, its magnitude reflects the problem scale in the original target units (Table~\ref{tab:real_metrics_z}). In particular, we obtain $\hat{\sigma}_\varepsilon^2(z)\approx 4.4\times 10^{-4}$, corresponding to $\hat{\sigma}_\varepsilon(z)\approx 0.021$.

\begin{table}[!ht]
\centering
\caption{MAP estimates of the model parameters on the real-world datasets under $k$-means anchor selection.}
\label{tab:real_map_params_50eps}
\resizebox{\textwidth}{!}{%
\begin{tabular}{lrrrrrrrrr}
\toprule
Case & $n$ & $I$ & $\hat\sigma_b^2$ & $\hat\sigma_v^2$ & $\hat\sigma_u^2$ & $\hat\sigma_a^2$ & $\hat\alpha$ & $\hat w$ & $\hat\sigma_\varepsilon^2$ \\
\midrule
S1\_super\_full\_all & 21263 & 79 & 0.983783 & 1.031285 & 1.031180 & 0.969412 & 0.487748 & 0.512268 & 5.1759e-01 \\
R1\_year\_full\_10k  & 10000 & 90 & 1.028279 & 1.032023 & 1.031995 & 0.968947 & 0.487489 & 0.512480 & 5.4418e-02 \\
R2\_year\_full\_20k  & 20000 & 90 & 1.021539 & 1.031795 & 1.031790 & 0.969213 & 0.487546 & 0.512372 & 5.3020e-02 \\
R3\_year\_full\_50k  & 50000 & 90 & 1.027704 & 1.032048 & 1.032012 & 0.968841 & 0.487509 & 0.512499 & 5.3630e-02 \\
\bottomrule
\end{tabular}%
}
\end{table}

\begin{table}[!ht]
\centering
\caption{Test-set performance metrics reported on the original target scale and total time on the real datasets.}
\label{tab:real_metrics_50eps}
\resizebox{\textwidth}{!}{%
\begin{tabular}{llrrrrrrr}
\toprule
Case & Dataset & $n$ & $I$ & Total time (s) & MAE & RMSE & MESE & SDESE \\
\midrule
S1\_super\_full\_all & superconduct & 21263 & 79 & 1.521 & 11.996 & 18.464 & 347.035 & 808.186 \\
R1\_year\_full\_10k  & yearpredictionmsd & 10000 & 90 & 2.141 & 8.151 & 11.346 & 129.177 & 276.568 \\
R2\_year\_full\_20k  & yearpredictionmsd & 20000 & 90 & 1.593 & 8.483 & 11.578 & 134.426 & 272.604 \\
R3\_year\_full\_50k  & yearpredictionmsd & 50000 & 90 & 2.266 & 8.551 & 11.614 & 135.163 & 292.555 \\
\bottomrule
\end{tabular}%
}
\end{table}

\begin{table}[!ht]
\centering
\caption{Scaled metrics normalized by the training-target standardization, relative to Table~\ref{tab:real_metrics_50eps}.}
\label{tab:real_metrics_z}
\footnotesize
\begin{tabularx}{\textwidth}{l*{5}{>{\centering\arraybackslash}X}}
\toprule
Case & $\mathrm{sd}(y_{\text{train}})$ & MAE$_z$ & RMSE$_z$ & $\hat\sigma_{\varepsilon}^2(z)$ & $\hat\sigma_{\varepsilon}^2(\text{orig})$\\
\midrule
S1\_super\_full\_all & 34.231 & 0.350 & 0.539 & 4.417e-04 & 5.1759e-01 \\
R1\_year\_full\_10k  & 11.104 & 0.734 & 1.022 & 4.414e-04 & 5.4418e-02 \\
R2\_year\_full\_20k  & 10.957 & 0.774 & 1.057 & 4.416e-04 & 5.3020e-02 \\
R3\_year\_full\_50k  & 11.022 & 0.776 & 1.054 & 4.415e-04 & 5.3630e-02 \\
\bottomrule
\end{tabularx}
\end{table}

Predictively, we observe only marginal changes as $n$ increases in \textit{YearPredictionMSD} while keeping $r=500$: RMSE remains around $11.35$--$11.61$ (original scale), suggesting that, under a fixed-rank approximation and a fixed training budget, part of the error becomes dominated by the low-rank Nystr\"om approximation. Even so, computational cost remains low (total time $\approx 1.5$--$2.3$ s per setting), highlighting scalability and reproducibility on large tabular datasets.

Finally, because the targets are measured in different natural units (years in \textit{YearPredictionMSD} and \textit{Kelvin} in \textit{Superconductivity}), absolute MAE/RMSE values on the original scale are not directly comparable across datasets. Therefore, in addition to reporting metrics on the original scale, we also report versions normalized by the training-target scale (MAE$_z$, RMSE$_z$) in Table~\ref{tab:real_metrics_z}, which allow comparisons of relative error magnitudes across settings and datasets.

\begin{table}[!ht]
\centering
\caption{Full YearPredictionMSD: MAP estimates of the model parameters under Nyström, using \textit{first} anchors, for different ranks $r$.}
\label{tab:year_full_map_params_gridN}
\renewcommand{\arraystretch}{0.90} 
\setlength{\tabcolsep}{11pt}
\resizebox{\textwidth}{!}{%
\begin{tabular}{rrrrrrrr}
\toprule
$r$ & $\hat{\sigma}_b^2$ & $\hat{\sigma}_v^2$ & $\hat{\sigma}_u^2$ & $\hat{\sigma}_a^2$ & $\hat{\alpha}$ & $\hat{w}$ & $\hat{\sigma}^2_{\varepsilon}$ \\
\midrule
1000 & 0.999745 & 1.031698 & 1.031808 & 0.968802 & 0.487477 & 0.512471 & $\,5.280451\times 10^{-2}\,$ \\
1500 & 1.001928 & 1.030765 & 1.031028 & 0.969650 & 0.487674 & 0.512086 & $\,5.281479\times 10^{-2}\,$ \\
2000 & 1.000840 & 1.031140 & 1.031370 & 0.969235 & 0.487561 & 0.512327 & $\,5.281858\times 10^{-2}\,$ \\
2500 & 1.000225 & 1.031369 & 1.031605 & 0.968976 & 0.487469 & 0.512463 & $\,5.281404\times 10^{-2}\,$ \\
3000 & 0.999901 & 1.031510 & 1.031739 & 0.968839 & 0.487441 & 0.512496 & $\,5.282333\times 10^{-2}\,$ \\
3500 & 0.999993 & 1.031638 & 1.031865 & 0.968724 & 0.487422 & 0.512520 & $\,5.282105\times 10^{-2}\,$ \\
4000 & 1.000136 & 1.031792 & 1.032011 & 0.968610 & 0.487406 & 0.512538 & $\,5.282208\times 10^{-2}\,$ \\
4500 & 0.999972 & 1.017239 & 1.018181 & 0.978938 & 0.492810 & 0.523093 & $\,5.280713\times 10^{-2}\,$ \\
5000 & 1.000011 & 1.022378 & 1.022732 & 0.973577 & 0.490332 & 0.518135 & $\,5.280737\times 10^{-2}\,$ \\
\bottomrule
\end{tabular}%
}
\end{table}

At this stage, we evaluate the GP model with the mixed kernel $K_{\text{mix}}$ on the full \textit{YearPredictionMSD} dataset (UCI), using a random split with $10\%$ held out for testing and the same preprocessing adopted in the other real-data experiments: min--max rescaling of $X$ to $[0,1]^I$ on the training set (and applied to the test set), centering $\tilde X=X-0.5$, and standardizing the target on the training set, with predictions de-standardized to report metrics on the original scale. We fit the model for each value of the rank $r$, keeping the anchor selection fixed to the \textit{first} strategy, in order to isolate the effect of $r$ on the quality of the Nyström approximation. We choose \textit{first} on the full dataset for scalability, since selecting anchors via \textit{$k$-means} would introduce an additional preprocessing cost that becomes significant at this scale.

\begin{table}[!ht]
\centering
\caption{Full YearPredictionMSD: test-set performance on the original target scale and total time under Nyström with \textit{first} anchors and different ranks $r$.}
\label{tab:year_full_metrics_gridN}
\footnotesize
\begin{tabularx}{\textwidth}{r*{5}{>{\centering\arraybackslash}X}}
\toprule
$r$ & Total time (s) & MAE & RMSE & MESE & SDESE \\
\midrule
1000 & 5.284   & 7.446 & 10.262 & 105.416 & 231.199 \\
1500 & 10.218  & 7.179 & 10.112 & 102.348 & 247.560 \\
2000 & 16.643  & 7.087 & 9.947  & 98.988  & 248.515 \\
2500 & 28.760  & 7.032 & 10.057 & 101.158 & 276.705 \\
3000 & 41.467  & 6.986 & 9.690  & 93.922  & 231.058 \\
3500 & 61.748  & 6.843 & 9.648  & 93.231  & 236.175 \\
4000 & 92.689  & 6.903 & 9.594  & 92.119  & 239.006 \\
4500 & 139.815 & 7.778 & 10.057 & 101.152 & 283.829 \\
5000 & 179.735 & 7.198 & 9.677  & 93.660  & 247.062 \\
\bottomrule
\end{tabularx}
\end{table}

\begin{table}[H]
\centering
\caption{Full YearPredictionMSD: metrics normalized by the training-target scale from Table~\ref{tab:year_full_metrics_gridN}.}
\label{tab:year_full_metrics_z_gridN}
\footnotesize
\begin{tabularx}{\textwidth}{r*{5}{>{\centering\arraybackslash}X}}
\toprule
$r$ & $\mathrm{sd}(y_{\text{train}})$ & MAE$_z$ & RMSE$_z$ & $\hat\sigma_{\varepsilon}^2(z)$ & $\hat\sigma_{\varepsilon}^2(\text{orig})$\\
\midrule
1000 & 10.936 & 0.681 & 0.938 & $\,4.414865\times 10^{-4}\,$ & $\,5.280451\times 10^{-2}\,$ \\
1500 & 10.936 & 0.656 & 0.925 & $\,4.415726\times 10^{-4}\,$ & $\,5.281479\times 10^{-2}\,$ \\
2000 & 10.936 & 0.648 & 0.910 & $\,4.416462\times 10^{-4}\,$ & $\,5.281858\times 10^{-2}\,$ \\
2500 & 10.936 & 0.643 & 0.920 & $\,4.415373\times 10^{-4}\,$ & $\,5.281404\times 10^{-2}\,$ \\
3000 & 10.936 & 0.639 & 0.886 & $\,4.416178\times 10^{-4}\,$ & $\,5.282333\times 10^{-2}\,$ \\
3500 & 10.936 & 0.626 & 0.882 & $\,4.416175\times 10^{-4}\,$ & $\,5.282105\times 10^{-2}\,$ \\
4000 & 10.936 & 0.631 & 0.877 & $\,4.416362\times 10^{-4}\,$ & $\,5.282208\times 10^{-2}\,$ \\
4500 & 10.936 & 0.711 & 0.920 & $\,4.414865\times 10^{-4}\,$ & $\,5.280713\times 10^{-2}\,$ \\
5000 & 10.936 & 0.658 & 0.885 & $\,4.415092\times 10^{-4}\,$ & $\,5.280737\times 10^{-2}\,$ \\
\bottomrule
\end{tabularx}
\end{table}

Tables~\ref{tab:year_full_map_params_gridN}--\ref{tab:year_full_metrics_gridN} highlight two main findings. First, the MAP estimates of the $K_{\text{mix}}$ parameters are stable as $r$ varies. In particular, $\hat{\alpha}\approx 0.487$ and $\hat{w}\approx 0.512$ remain nearly constant, indicating that the model preserves a balanced mixture of the components and does not collapse to the degenerate cases ($w\in\{0,1\}$).

Similarly, the estimated \emph{nugget} remains essentially constant across the grid of $r$, showing variation only in the 5th--6th decimal place when reported with higher precision. On the full \textit{YearPredictionMSD} dataset, $\hat\sigma_\varepsilon^2(\text{orig})$ stays within $[5.280451,\,5.282333]\times 10^{-2}$, which corresponds to a relative variation of approximately $0.036\%$ (Table~\ref{tab:year_full_metrics_z_gridN}). This stability is consistent with interpreting $\sigma_\varepsilon^2$ as i.i.d.\ noise: as $r$ changes, it primarily affects the Nyström approximation's ability to represent the correlated component of the latent process, while the estimated observational noise level remains essentially unchanged. Residual differences can be attributed to numerical/optimization effects and rounding.


In addition to assessing the stability of the fit as \(r\) varies, Table~\ref{tab:year_full_metrics_gridN} can be interpreted operationally as a direct guide for choosing the Nyström rank under a total time budget \(T\) (training + prediction). For a fixed budget, a simple rule is to select \(r_{\max}(T)\), defined as the largest rank in the grid such that \(\mathrm{time}(r)\le T\), which makes explicit the cost of increasing the capacity of the approximator. For instance, under \(T=30\)s, the largest feasible rank in the grid is \(r=2500\) (28.760s), with MAE \(=7.032\) and RMSE \(=10.057\); under \(T=60\)s, \(r_{\max}(T)=3000\) (41.467s), with MAE \(=6.986\) and RMSE \(=9.690\); and under \(T=300\)s, \(r_{\max}(T)=5000\) (179.735s), with MAE \(=7.198\) and RMSE \(=9.677\).

At the same time, the error is not strictly monotone in \(r\): some intermediate values perform worse than smaller ranks (e.g., \(r=2500\) compared with \(r=2000\)), and there is also a point with a marked degradation (\(r=4500\)). This behavior is consistent with numerical variability and sensitivity of training/optimization when combined with a low-rank approximation, and it reinforces the trade-off interpretation with diminishing returns: increasing \(r\) raises cost, but does not guarantee an immediate improvement in MAE/RMSE.

Therefore, a second practical rule is to choose, within each budget \(T\), the rank \(r_{\text{best}}(T)\) that minimizes the test RMSE. Table~\ref{tab:budget_r_choice} summarizes both choices. In particular, for \(T=60\)s, \(r=3000\) emerges as a competitive point by combining low cost with reduced RMSE; for a more generous budget (\(T=300\)s), the lowest RMSE in the grid occurs at \(r=4000\) (92.689s), still well below the cost of \(r=5000\), indicating that additional gains beyond \(r\in[3000,4000]\) tend to be marginal relative to the increase in runtime.

\begin{table}[!ht]
\centering
\caption{Choice of \(r\) on the full YearPredictionMSD dataset under total runtime budgets \(T\). We report \(r_{\max}(T)\) (the largest \(r\) such that \(\mathrm{time}(r)\le T\)) and, additionally, the \(r\) that yields the lowest RMSE within the same budget.}
\label{tab:budget_r_choice}
\footnotesize
\setlength{\tabcolsep}{6pt} 
\renewcommand{\arraystretch}{1.15}
\begin{tabular*}{\linewidth}{@{\extracolsep{\fill}} l r r r r r r @{}}
\toprule
Roof \(T\) & \(r_{\max}(T)\) & Total time (s) & MAE & RMSE & \(r_{\text{best}}(T)\) & RMSE\(_{\text{best}}\) \\
\midrule
30s  & 2500 & 28.760  & 7.032 & 10.057 & 2000 & 9.947 \\
60s  & 3000 & 41.467  & 6.986 & 9.690  & 3000 & 9.690 \\
300s & 5000 & 179.735 & 7.198 & 9.677  & 4000 & 9.594 \\
\bottomrule
\end{tabular*}
\end{table}

In summary, the full-dataset study confirms the Nyström rank as the main driver of the cost--accuracy trade-off: larger $r$ values tend to reduce error up to a saturation regime, while the kernel parameters and the nugget remain essentially stable. Considering accuracy and computational cost jointly, $r=4000$ yields the lowest RMSE among the evaluated values, although configurations such as $r\in[3000,4000]$ are also competitive at a substantially lower cost than $r=5000$.

Overall, on the real datasets, MAP estimates remain consistent across sample sizes, and performance confirms that the scalable approximation is the primary source of the time--error trade-off. In the large-scale study, the approximation complexity directly controls this trade-off, with gains that tend to saturate beyond an intermediate range and improvements concentrated in intermediate regimes.

\section{Conclusion}
\label{sec:conclusao}
This work strengthens the connection between infinite-width limits of BNNs and GPs, and leverages this bridge to propose and evaluate a mixed kernel \(K_{\text{mix}}\), interpreted as a combination of two complementary process behaviors. From a theoretical standpoint, the formulation establishes conditions that guarantee fundamental properties of the proposed class and discusses regimes of non-identifiability in finite samples, motivating a reduced parameterization that is more stable and interpretable. From a computational standpoint, we develop a scalable procedure based on MAP estimation with a \textit{Nyström} approximation, enabling efficient training and prediction with explicit uncertainty quantification.

Empirically, simulation studies and experiments on real-world datasets support three main findings: (i) MAP fitting is stable across scenarios, with the mixing parameter \(w\) consistently estimated and showing no systematic collapse to degenerate cases; (ii) anchor selection in the \textit{Nyström} approximation affects the cost--accuracy trade-off, with more noticeable gains in more complex regimes; and (iii) at scale, the approximation rank directly governs the time--performance trade-off, with diminishing returns beyond an intermediate range. These results indicate that the mixing in \(K_{\text{mix}}\) provides practical robustness by avoiding losses associated with fixed \emph{kernel} choices, while retaining a small number of parameters with clear probabilistic interpretation.

Despite these advances, the study has limitations. First, the experimental evaluation was conducted over a finite set of scenarios and configuration choices, so other regimes may exhibit different behavior. Second, we focus on univariate regression with \emph{i.i.d.} Gaussian noise, which limits immediate applicability to multivariate, heteroscedastic, or classification settings. Third, while \textit{Nyström} provides scalability, performance depends both on anchor selection and on the complexity of the low-rank approximation, and a systematic sensitivity analysis of these factors was not explored exhaustively.

As future work, we aim to extend the model to multivariate outputs and non-Gaussian likelihoods, including classification. We also plan to investigate hierarchical variants inspired by the BNN \(\rightarrow\) GP bridge to capture local heterogeneities, broaden empirical validation across additional domains and scaling regimes with more systematic sensitivity analyses, and compare alternative scalable approximations and inference strategies beyond MAP, with the goal of improving calibration and robustness under controlled computational budgets.

\section*{Acknowledgments}
We thank PPGEst--UFMG for a collaborative research environment and UEMG for institutional support. This work was supported by CNPq, FAPEMIG, and CAPES.

\bibliographystyle{abntex2-alf}
\bibliography{references} 

\appendix

\section{Technical details and proofs}
\label{apx:technical_details}

\subsection{Pre-activation moments and dependence across inputs}
\label{apx:preactivation_moments}

For a fixed hidden unit \(j\), consider the pre-activations
\(
z(\mathbf{x}) = a + \mathbf{u}^\top \mathbf{x},
z(\mathbf{x}') = a + \mathbf{u}^\top \mathbf{x}',
\)
where \(a\) and \(\mathbf{u}\) are drawn from the parameter priors (with zero mean and finite variances).
Under the assumptions in the main text, we have:

\paragraph{Mean.}
\[
\mathbb{E}[z(\mathbf{x})] = \mathbb{E}[a] + \sum_{i=1}^I x_i \mathbb{E}[u_i] = 0,
\qquad
\mathbb{E}[z(\mathbf{x}')] = 0.
\]

\paragraph{Variance.}
Assuming independence between \(a\) and \(\mathbf{u}\), and finite variances,
\[
\mathrm{Var}\big(z(\mathbf{x})\big)
= \mathrm{Var}(a) + \mathrm{Var}(\mathbf{u}^\top \mathbf{x})
= \sigma_a^2 + \sigma_u^2 \,\mathbf{x}^\top \mathbf{x},
\]
and similarly \(\mathrm{Var}\big(z(\mathbf{x}')\big)=\sigma_a^2 + \sigma_u^2 \,\mathbf{x}'^\top \mathbf{x}'\).

\paragraph{Covariance.}
Note that \(z(\mathbf{x})\) and \(z(\mathbf{x}')\) are \emph{not} independent: they share the same \(a\) and the same \(\mathbf{u}\).
Their covariance is
\[
\mathrm{Cov}\big(z(\mathbf{x}),z(\mathbf{x}')\big)
= \mathrm{Var}(a) + \mathrm{Cov}(\mathbf{u}^\top\mathbf{x},\mathbf{u}^\top\mathbf{x}')
= \sigma_a^2 + \sigma_u^2 \,\mathbf{x}^\top\mathbf{x}'.
\]
In particular, when \(a\) and \(\mathbf{u}\) are Gaussian, \((z(\mathbf{x}),z(\mathbf{x}'))\) is a centered bivariate Gaussian
with covariance matrix determined by the quantities above.

\subsection{A general regularity condition}
\label{apx:regularity}

We next state a simple sufficient condition ensuring \(\mathbb{E}[h(Z)^2]<\infty\) when \(Z\) is Gaussian.

\begin{proposition}[Polynomial growth implies a finite second moment]
\label{prop:second_moment}
If \(h:\mathbb{R}\to\mathbb{R}\) satisfies \(|h(z)|\le C(1+|z|^p)\) for some \(C>0\) and \(p\ge0\),
then for \(Z\sim\mathcal{N}(0,\sigma^2)\) we have \(\mathbb{E}[h(Z)^2]<\infty\).
\end{proposition}

\begin{proof}
By assumption, \(h(Z)^2 \le 2C^2(1+|Z|^{2p})\). Since Gaussian random variables have moments of all orders,
\(\mathbb{E}[|Z|^{2p}]<\infty\). Hence,
\(\mathbb{E}[h(Z)^2] \le 2C^2(1+\mathbb{E}[|Z|^{2p}])<\infty\).
\end{proof}

\subsection{Proof of Theorem~\ref{thm:convergence_main}}
\label{apx:proof_thm}


For a fixed \(k\), the function
\(
f_k(\mathbf{x}) = b_k + \sum_{j=1}^{H} v_{kj}\, h\!\big(z_j(\mathbf{x})\big)
\)
is the sum of the bias term \(b_k\) and \(H\) terms of the form \(v_{kj}h(z_j(\mathbf{x}))\). Under the the prior independence assumptions, the terms \(v_{kj}h(z_j(\mathbf{x}))\) are independent across \(j\).

Define \(v_{kj}^{*} = \sqrt{H}\,v_{kj}\), so that \(v_{kj}^{*} \sim \mathcal{N}(0,\sigma_v^2)\). Then, by the CLT,
\[
\sqrt{H} \sum_{j=1}^{H} \frac{v_{kj}^{*}\, h(z_j(\mathbf{x}))}{H}
\;\Longrightarrow\;
\mathcal{N}\!\bigl(0,\,\sigma_v^2\,\mathbb{E}[h(z_j(\mathbf{x}))^2]\bigr).
\]

Independently of the specific distributions of \(a_j\), \(u_{ji}\), and \(v_{kj}\), as long as each term
\(v_{kj}h(z_j(\mathbf{x}))\) has mean zero and finite variance, the sum \(\sum_{j=1}^{H} v_{kj} h(z_j(\mathbf{x}))\)
remains a sum of \(H\) independent terms, due to the independence of \(v_{kj}\) and \(z_j\). Thus, for each \(\mathbf{x}\),
\begin{align}
\mathbb{E}[f_k(\mathbf{x})]
= \mathbb{E}[b_k] + \sum_{j=1}^{H} \mathbb{E}[v_{kj}]\,\mathbb{E}[h(z_j(\mathbf{x}))]
&= 0 + \sum_{j=1}^{H} 0 \cdot \mathbb{E}[h(z_j(\mathbf{x}))] = 0, \notag
\end{align}
and
\begin{align}
\mathrm{Var}[f_k(\mathbf{x})]
= \mathrm{Var}[b_k] + \sum_{j=1}^{H} \mathrm{Var}\!\left(v_{kj} h(z_j(\mathbf{x}))\right) = \sigma_b^2 + \sum_{j=1}^{H} \frac{\sigma_v^2}{H}\,\mathbb{E}[h(z_j(\mathbf{x}))^2] = \sigma_b^2 + \sigma_v^2\,\mathbb{E}[h(z_j(\mathbf{x}))^2]. \notag
\end{align}
Therefore,
\[
f_k(\mathbf{x})
\;\xRightarrow{d}\;
\mathcal{N}\!\Bigl(0,\,\sigma_b^2 + \sigma_v^2\,\mathbb{E}[h(z_j(\mathbf{x}))^2]\Bigr).
\]

Note that one must verify that second-order moments are finite; if \(h\) satisfies Proposition~\ref{prop:second_moment},
then the theorem’s assumptions guarantee that \(\mathrm{Var}[f_k(\mathbf{x})] < \infty\).

Now consider \(n\) input points \((\mathbf{x}_1,\mathbf{x}_2,\ldots,\mathbf{x}_n)\).
For each \(p\in\{1,\ldots,n\}\), \(f_k(\mathbf{x}_p)\) is a sum of terms \(v_{kj}h(z_j(\mathbf{x}_p))\) plus a bias, so
\(f_k(\mathbf{x}_p)\sim \mathcal{N}\!\bigl(0,\sigma_b^2+\sigma_v^2\mathbb{E}[h(z_j(\mathbf{x}_p))^2]\bigr)\), i.e.,
\(\mathbf{f}\sim \mathcal{N}_n(\mathbf{0},\mathbf{K})\).
Note that for different \(\mathbf{x}_p\) and \(\mathbf{x}_{p'}\), the terms \(v_{kj}h(z_j(\mathbf{x}_p))\) and
\(v_{kj}h(z_j(\mathbf{x}_{p'}))\) are correlated because they share the same \(v_{kj}\). Moreover, since \(v_{kj}\) and
\(h(z_j(\mathbf{x}_{p'}))\) are independent,
\begin{align}
\mathrm{Cov}[f_k(\mathbf{x}_p), f_k(\mathbf{x}_{p'})]
&= \mathbb{E}[f_k(\mathbf{x}_p) f_k(\mathbf{x}_{p'})] - \mathbb{E}[f_k(\mathbf{x}_p)] \mathbb{E}[f_k(\mathbf{x}_{p'})] \notag \\
&= \mathbb{E}[f_k(\mathbf{x}_p) f_k(\mathbf{x}_{p'})] - 0 \notag \\
&= \mathrm{Var}[b_k] + \sum_{j=1}^{H} \mathrm{Cov}\!\left(v_{kj} h(z_j(\mathbf{x}_p)),\, v_{kj} h(z_j(\mathbf{x}_{p'}))\right). \notag
\end{align}
Since \(v_{kj}\) are independent across different \(j\),
\begin{align}
\sum_{j=1}^{H} \mathrm{Cov}\!\left(v_{kj} h(z_j(\mathbf{x}_p)), v_{kj} h(z_j(\mathbf{x}_{p'}))\right)
&= \sum_{j=1}^{H} \mathbb{E}[v_{kj}^2]\, \mathbb{E}[h(z_j(\mathbf{x}_p)) h(z_j(\mathbf{x}_{p'}))] \notag \\
&= \sum_{j=1}^{H} \frac{\sigma_v^2}{H}\,\mathbb{E}[h(z_j(\mathbf{x}_p)) h(z_j(\mathbf{x}_{p'}))] \notag \\
&= \sigma_v^2\,\mathbb{E}[h(z(\mathbf{x}_p)) h(z(\mathbf{x}_{p'}))].
\end{align}
Therefore,
\begin{align}
\mathrm{Cov}[f_k(\mathbf{x}_p), f_k(\mathbf{x}_{p'})]
&= \sigma_b^2 + \sigma_v^2\,\mathbb{E}[h(z(\mathbf{x}_p)) h(z(\mathbf{x}_{p'}))] \notag \\
&= \sigma_b^2 + \sigma_v^2 \int h(z(\mathbf{x}_p))\,h(z(\mathbf{x}_{p'})) \cdot p\!\big(z_j(\mathbf{x}_p),z_j(\mathbf{x}_{p'})\big)\,dz\,dz' \notag \\
&= \sigma_b^2 + \sigma_v^2\,C(\mathbf{x}_{p},\mathbf{x}_{p'}) \notag \\
&= K(\mathbf{x}_{p},\mathbf{x}_{p'}). \notag
\end{align}

The fact that, for any finite set of points \(\{\mathbf{x}_1,\mathbf{x}_2,\ldots,\mathbf{x}_n\}\), the vector
\((f_k(\mathbf{x}_1),\ldots,f_k(\mathbf{x}_n))\) converges to a multivariate Normal distribution implies that \(f_k\)
converges to a GP with zero mean and the covariance defined above, namely
\[
\mathbf{K} \;=\; \sigma_b^2 \,\mathbf 1_n\mathbf 1_n^\top \;+\; \sigma_v^2\,\mathbf{C}.
\]

\subsection{Proof of Proposition~\ref{prop:convergencia_mista}}
\label{apx:proof_prop_convergencia_mista}

\paragraph{Step 0 (setup and independence).}
We work under the \emph{additive (block) mixture construction} described in the main text. For each component
$m=1,\ldots,M$, the network uses a separate collection of output weights $(v^{(m)}_j)_{j=1}^H$, and we assume:
\[
v^{(m)}_j \ \text{are mutually independent in } j,\ \text{independent across } m,
\qquad
v^{(m)}_j \overset{\text{iid}}{\sim}\mathcal N\!\left(0,\frac{\sigma_v^2}{H}\right).
\]
The pre-activations are
\[
z_j(\mathbf x)=a_j+\mathbf u_j^\top \mathbf x,
\qquad j=1,\dots,H,
\]
where $(a_j,\mathbf u_j)$ are i.i.d. in $j$ and independent of all $v^{(m)}_j$. We also assume $b\sim\mathcal N(0,\sigma_b^2)$
independent of all other parameters.

\paragraph{Step 1 (network output).}
For a single-output network, the block-mixture prior is
\[
f(\mathbf x)
=
b+\sum_{m=1}^M \sqrt{w_m}\sum_{j=1}^H v^{(m)}_j\,h_m\!\bigl(z_j(\mathbf x)\bigr),
\qquad
w_m\in(0,1),\ \ \sum_{m=1}^M w_m=1.
\]
(For multiple outputs $k$, the argument applies componentwise to $f_k$.)

\paragraph{Step 2 (expand the covariance).}
For two inputs $\mathbf x_i,\mathbf x_{i'}$, by bilinearity of covariance,
\begin{align*}
\mathrm{Cov}\!\bigl[f(\mathbf x_i),f(\mathbf x_{i'})\bigr]
&= \mathrm{Var}(b) \\
&\quad+\mathrm{Cov}\!\left(\sum_{m=1}^M \sqrt{w_m}\sum_{j=1}^H v^{(m)}_j h_m(z_j(\mathbf x_i)),\;
\sum_{m'=1}^M \sqrt{w_{m'}}\sum_{\ell=1}^H v^{(m')}_\ell h_{m'}(z_\ell(\mathbf x_{i'}))\right).
\end{align*}

\paragraph{Step 3 (cross terms vanish).}
All terms with $m\neq m'$ vanish because the blocks $\{v^{(m)}_j\}$ and $\{v^{(m')}_\ell\}$ are independent and centered.
Likewise, for fixed $m$, all terms with $j\neq \ell$ vanish because $v^{(m)}_j$ are independent across $j$ and centered.
Therefore,
\[
\mathrm{Cov}\!\bigl[f(\mathbf x_i),f(\mathbf x_{i'})\bigr]
=
\sigma_b^2
+
\sum_{m=1}^M w_m \sum_{j=1}^H 
\mathrm{Cov}\!\left(v^{(m)}_j h_m(z_j(\mathbf x_i)),\,v^{(m)}_j h_m(z_j(\mathbf x_{i'}))\right).
\]

\paragraph{Step 4 (factor each remaining covariance).}
Since $v^{(m)}_j$ is independent of $z_j(\cdot)$ and $\mathbb E[v^{(m)}_j]=0$,
\[
\mathrm{Cov}\!\left(v^{(m)}_j h_m(z_j(\mathbf x_i)),\,v^{(m)}_j h_m(z_j(\mathbf x_{i'}))\right)
=
\mathbb E[(v^{(m)}_j)^2]\,
\mathbb E\!\big[h_m(z_j(\mathbf x_i))\,h_m(z_j(\mathbf x_{i'}))\big].
\]
With the scaling $\mathbb E[(v^{(m)}_j)^2]=\sigma_v^2/H$, we obtain
\[
\mathrm{Cov}\!\bigl[f(\mathbf x_i),f(\mathbf x_{i'})\bigr]
=
\sigma_b^2
+
\sum_{m=1}^M w_m \sum_{j=1}^H 
\frac{\sigma_v^2}{H}\,
\mathbb E\!\big[h_m(z_j(\mathbf x_i))\,h_m(z_j(\mathbf x_{i'}))\big].
\]

\paragraph{Step 5 (simplify the sum over hidden units).}
Because $(a_j,\mathbf u_j)$ are i.i.d. in $j$, the pair $\bigl(z_j(\mathbf x_i),z_j(\mathbf x_{i'})\bigr)$ has the same distribution
for every $j$, so
\[
\mathbb E\!\big[h_m(z_j(\mathbf x_i))\,h_m(z_j(\mathbf x_{i'}))\big]
=
\mathbb E\!\big[h_m(Z)\,h_m(Z')\big]
=: K_m(\mathbf x_i,\mathbf x_{i'}),
\]
where $(Z,Z')$ denotes the pre-activation pair induced by the priors at inputs $(\mathbf x_i,\mathbf x_{i'})$.
Thus,
\[
\sum_{j=1}^H 
\frac{\sigma_v^2}{H}\,
\mathbb E\!\big[h_m(z_j(\mathbf x_i))\,h_m(z_j(\mathbf x_{i'}))\big]
=
H\cdot \frac{\sigma_v^2}{H}\,K_m(\mathbf x_i,\mathbf x_{i'})
=
\sigma_v^2\,K_m(\mathbf x_i,\mathbf x_{i'}).
\]
Substituting back yields the kernel form
\[
\mathrm{Cov}\!\bigl[f(\mathbf x_i),f(\mathbf x_{i'})\bigr]
=
\sigma_b^2 + \sigma_v^2 \sum_{m=1}^M w_m\,K_m(\mathbf x_i,\mathbf x_{i'}).
\]

\paragraph{Step 6 (GP limit and positive semidefiniteness).}
For any finite set of inputs, $f(\mathbf x)$ is a sum (over $j$) of independent terms with variance controlled by the
$1/H$ scaling. Under the same moment assumptions as in Theorem~\ref{thm:convergence_main} (e.g., $\mathbb E[h_m(Z)^2]<\infty$),
the multivariate CLT implies that the finite-dimensional distributions converge to a multivariate Gaussian with covariance matrix
built from
\[
K(\mathbf x,\mathbf x')=\sigma_b^2 + \sigma_v^2 \sum_{m=1}^M w_m\,K_m(\mathbf x,\mathbf x').
\]
Hence, as $H\to\infty$, $f(\cdot)$ converges (in the sense of finite-dimensional distributions) to a GP with covariance $K$.
Finally, since each $K_m$ is symmetric positive semidefinite and $w_m\ge 0$, the kernel $K$ is also symmetric positive semidefinite.

\subsection{Proof of Proposition~\ref{prop:ident_mix}}
\label{apx:ident_mix_proof}

\begin{proof}
Suppose there exist $\theta_1\neq\theta_2$ in $\Theta$ such that
\[
K_{\text{mix}}(\mathbf{x},\mathbf{x}';\theta_1)\equiv K_{\text{mix}}(\mathbf{x},\mathbf{x}';\theta_2),
\qquad \forall\,(\mathbf{x},\mathbf{x}')\in\mathbb{R}^d\times\mathbb{R}^d.
\]
Write
\[
K_{\text{mix}}(\mathbf{x},\mathbf{x}';\theta)
=\sigma_b^2+\sigma_v^2\Big[w\,A(\mathbf{x},\mathbf{x}')
+(1-w)\,B(\mathbf{x},\mathbf{x}';\alpha)\Big],
\]
where $A$ denotes the smooth (arcsine) term and $B(\cdot;\alpha)$ denotes the angular (ReLU/LeakyReLU) term, both depending on
$(\mathbf{x},\mathbf{x}')$ through $\rho(\mathbf{x},\mathbf{x}')$ and through $\sigma_z(\mathbf{x}),\sigma_{z'}(\mathbf{x}')$,
which are determined by $(\sigma_a^2,\sigma_u^2)$.

\paragraph{Step 1: identification of $\sigma_b^2$.}
From the identity $K_{\text{mix}}(\cdot,\cdot;\theta_1)\equiv K_{\text{mix}}(\cdot,\cdot;\theta_2)$ it follows that
\[
\sigma_{b,1}^2-\sigma_{b,2}^2
=
F_{\theta_2}(\mathbf{x},\mathbf{x}')-F_{\theta_1}(\mathbf{x},\mathbf{x}'),
\quad \forall(\mathbf{x},\mathbf{x}'),
\]
where $F_{\theta}:=\sigma_v^2\big[wA+(1-w)B(\cdot;\alpha)\big]$.
The left-hand side is constant. Fixing $\|\mathbf{x}\|=\|\mathbf{x}'\|=r$ and choosing different angles between $\mathbf{x}$ and $\mathbf{x}'$
(which yields different values of $\rho$), both $A$ and $B$ vary; hence the difference on the right-hand side cannot be constant over the whole domain
unless $F_{\theta_1}\equiv F_{\theta_2}$. Consequently, $\sigma_{b,1}^2=\sigma_{b,2}^2$ and
\[
F_{\theta_1}(\mathbf{x},\mathbf{x}')\equiv F_{\theta_2}(\mathbf{x},\mathbf{x}').
\]

\paragraph{Step 2: non-proportionality of $A$ and $B(\cdot;\alpha)$.}
Restrict attention to pairs with $\|\mathbf{x}\|=\|\mathbf{x}'\|=r$, so that $\sigma_z=\sigma_{z'}$ is fixed and $\rho$ varies in $(-1,1)$.
Within this family, $A(\rho)$ is real-analytic on $(-1,1)$ (since it involves $\arcsin$), whereas $B(\rho;\alpha)$ contains terms such as
$\sqrt{1-\rho^2}$ and $\arccos(\rho)$, with qualitatively different behavior as $|\rho|\to 1$.
Therefore, $A$ and $B(\cdot;\alpha)$ cannot be proportional as functions of $\rho$ over the entire interval $(-1,1)$.

\paragraph{Step 3: identification of $(\sigma_a^2,\sigma_u^2)$.}
Now consider orthogonal pairs $\mathbf{x}\perp\mathbf{x}'$ with $\|\mathbf{x}\|=\|\mathbf{x}'\|=r>0$.
Then
\[
\rho(r)=\frac{\sigma_a^2}{\sigma_a^2+\sigma_u^2 r^2},
\qquad
\sigma_z^2(r)=\sigma_{z'}^2(r)=\sigma_a^2+\sigma_u^2 r^2.
\]
The identity $F_{\theta_1}\equiv F_{\theta_2}$ for all $r$ implies that the functions $r\mapsto \rho(r)$ and $r\mapsto \sigma_z^2(r)$
coincide under $\theta_1$ and $\theta_2$. Since $\sigma_z^2(r)$ is affine in $r^2$ and $\rho(r)$ depends only on the pair
$(\sigma_a^2,\sigma_u^2)$ through a rational function of $r^2$, it follows that
\[
(\sigma_{a,1}^2,\sigma_{u,1}^2)=(\sigma_{a,2}^2,\sigma_{u,2}^2).
\]

\paragraph{Step 4: identification of $\sigma_v^2,w,\alpha$.}
With $(\sigma_a^2,\sigma_u^2)$ identified, the quantities $\rho(\mathbf{x},\mathbf{x}')$, $\sigma_z(\mathbf{x})$, and $\sigma_{z'}(\mathbf{x}')$
match on both sides. Since $A$ and $B(\cdot;\alpha)$ are not proportional in $\rho$, the equality
\(
F_{\theta_1}\equiv F_{\theta_2}
\)
forces the coefficients to coincide:
\[
\sigma_{v,1}^2=\sigma_{v,2}^2,\qquad w_1=w_2,\qquad \alpha_1=\alpha_2.
\]
Together with $\sigma_{b,1}^2=\sigma_{b,2}^2$, we obtain $\theta_1=\theta_2$, contradicting $\theta_1\neq\theta_2$.
Hence, the mapping $\theta\mapsto K_{\text{mix}}(\cdot,\cdot;\theta)$ is injective in the interior of $\Theta$.

In boundary cases ($w\in\{0,1\}$ or $\alpha\in\{0,1\}$), some parameters may drop out of the \textit{kernel}
and identifiability may fail; these cases are excluded by assumption.
\end{proof}

\end{document}